\documentclass[10pt]{article} 

\usepackage[preprint]{rlj} 

\usepackage{booktabs}       
\usepackage{amsfonts}       
\usepackage{nicefrac}       
\usepackage{microtype}      
\usepackage{graphicx}       
\usepackage{amsmath}
\usepackage{algorithm}
\usepackage[noend]{algpseudocode}
\usepackage{mathrsfs}
\usepackage{dsfont}
\usepackage{array}
\usepackage{bbm}
\usepackage[mathscr]{eucal}
\usepackage{psfrag}
\usepackage{here}
\usepackage{wasysym}
\usepackage[dvipsnames, table]{xcolor}  
\usepackage{caption}                     
\usepackage{subcaption}
\usepackage{amsthm}
\usepackage{todonotes}
\usepackage{tabulary}
\usepackage{outlines}
\usepackage{thmtools, thm-restate}

\newtheorem{proposition} {Proposition}
\newtheorem{lemma} {Lemma}

\definecolor{ian_highlight}{RGB}{100, 2, 2}

\errorcontextlines\maxdimen

\makeatletter
    \newcommand*{\algrule}[1][\algorithmicindent]{\makebox[#1][l]{\hspace*{.5em}\thealgruleextra\vrule height \thealgruleheight depth \thealgruledepth}}%
\newcommand*{\thealgruleextra}{}
\newcommand*{\thealgruleheight}{.75\baselineskip}
\newcommand*{\thealgruledepth}{.25\baselineskip}

\newcount\ALG@printindent@tempcnta
\def\ALG@printindent{%
    \ifnum \theALG@nested>0
        \ifx\ALG@text\ALG@x@notext
        \else
            \unskip
            \addvspace{-1pt}
            \ALG@printindent@tempcnta=1
            \loop
                \algrule[\csname ALG@ind@\the\ALG@printindent@tempcnta\endcsname]%
                \advance \ALG@printindent@tempcnta 1
            \ifnum \ALG@printindent@tempcnta<\numexpr\theALG@nested+1\relax
            \repeat
        \fi
    \fi
    }%
\usepackage{etoolbox}
\makeatother

\newbox\statebox
\newcommand{\myState}[1]{%
    \setbox\statebox=\vbox{#1}%
    \edef\thealgruleheight{\dimexpr \the\ht\statebox+1pt\relax}%
    \edef\thealgruledepth{\dimexpr \the\dp\statebox+1pt\relax}%
    \ifdim\thealgruleheight<.75\baselineskip
        \def\thealgruleheight{\dimexpr .75\baselineskip+1pt\relax}%
    \fi
    \ifdim\thealgruledepth<.25\baselineskip
        \def\thealgruledepth{\dimexpr .25\baselineskip+1pt\relax}%
    \fi
    \State #1%
    \def\thealgruleheight{\dimexpr .75\baselineskip+1pt\relax}%
    \def\thealgruledepth{\dimexpr .25\baselineskip+1pt\relax}%
}

\title{Delightful Policy Gradient}
\setrunningtitle{Delightful Policy Gradient}

\author{Ian Osband\textsuperscript{1}}
\emails{iosband@google.com}
\affiliations{
$^{1}$\textbf{Google DeepMind}
}

\contribution{
    We introduce the Delightful Policy Gradient (DG), which gates each sampled gradient term by a sigmoid of \emph{delight}, the product of advantage and action surprisal (negative log probability).
    DG amplifies rare successes and suppresses rare failures, yielding a simple drop-in replacement for standard policy gradients that requires no importance ratios.
    }{
    Prior variance-reduction methods (baselines, control variates~\citep{kool2019buy}) reduce variance but leave the expected gradient direction unchanged.
    Trust-region methods~\citep{schulman2015trust,schulman2017proximal} clip importance ratios to limit policy change, while advantage-weighted methods~\citep{peng2019advantage,abdolmaleki2024preference} reweight by advantage alone.
    Neither uses action surprisal to reshape the update.
    }

\contribution{
    In tabular $K$-armed bandits, we prove that DG improves policy-gradient updates through two distinct mechanisms.
    Within a single decision context, DG reduces directional variance by suppressing noise from low-probability negative-advantage actions (Proposition~1).
    Across multiple contexts, DG changes the bias of the expected gradient, shifting it strictly closer to the supervised cross-entropy oracle, even in the infinite-sample limit (Proposition~2).
    }{
    Standard policy gradients allocate gradient budget in proportion to current success probability, creating a self-reinforcing dynamic in which easy contexts dominate while harder ones stall \citep{williams1992simple}.
    DG instead redistributes budget toward harder contexts, identifying a more balanced allocation rule for long-run learning under bandit feedback.
    }

\contribution{
    Empirically, DG outperforms REINFORCE, PPO, and advantage-weighted baselines across MNIST, transformer sequence modeling, and continuous control.
    The gains grow with problem difficulty: on Token Reversal, DG exhibits a smaller empirical scaling exponent than all baselines.
    }{
    Together, these experiments show that DG's mechanism is not a tabular artifact: MNIST diagnoses gradient geometry, Token Reversal tests scaling in sequential learning with Transformers, and the DeepMind Control Suite~\citep{tassa2018deepmind} tests density-based surprisal in continuous action spaces.
    }

\keywords{policy gradient, reinforcement learning, gradient estimation, variance reduction, scaling}

\summary{
Standard policy gradients weight each sampled action by advantage alone, regardless of how likely that action was under the current policy.
This creates two pathologies: within a single decision context (e.g.\ one image or prompt), a rare negative-advantage action can disproportionately distort the update direction; across many such contexts in a batch, the expected gradient over-allocates budget to contexts the policy already handles well.
We introduce the \textit{Delightful Policy Gradient} (DG), which gates each term with a sigmoid of \emph{delight}, the product of advantage and action surprisal (negative log-probability).
For $K$-armed bandits, DG provably improves directional accuracy in a single context and, across multiple contexts, shifts the expected gradient strictly closer to the supervised cross-entropy oracle.
This second effect is not variance reduction: it persists even with infinite samples.
Empirically, DG outperforms REINFORCE, PPO, and advantage-weighted baselines across MNIST, transformer sequence modeling, and continuous control, with larger gains on harder tasks.
}

\begin{document}

\makeCover  
\maketitle

\begin{abstract}
Standard policy gradients weight each sampled action by advantage alone, regardless of how likely that action was under the current policy.
This creates two pathologies: within a single decision context (e.g.\ one image or prompt), a rare negative-advantage action can disproportionately distort the update direction; across many such contexts in a batch, the expected gradient over-allocates budget to contexts the policy already handles well.
We introduce the \textit{Delightful Policy Gradient} (DG), which gates each term with a sigmoid of \emph{delight}, the product of advantage and action surprisal (negative log-probability).
For $K$-armed bandits, DG provably improves directional accuracy in a single context and, across multiple contexts, shifts the expected gradient strictly closer to the supervised cross-entropy oracle.
This second effect is not variance reduction: it persists even with infinite samples.
Empirically, DG outperforms REINFORCE, PPO, and advantage-weighted baselines across MNIST, transformer sequence modeling, and continuous control, with larger gains on harder tasks.
\end{abstract}

\section{Introduction}
\label{sec:intro}

Policy gradient methods underpin some of the most consequential systems in modern AI, from superhuman game play~\citep{silver2017mastering,vinyals2019grandmaster} to large language model alignment~\citep{ouyang2022training}.
In deep learning, optimizers often normalize gradients or constrain effective step size~\citep{kingma2014adam,you2019large}, making update direction a primary determinant of learning progress.
Most prior work therefore asks how to estimate the policy-gradient direction with lower variance or follow it more safely.
We ask a different question: is the policy-gradient direction itself the right one to follow?

We show that, in many settings, it is not.
Policy gradients weight each per-sample term by advantage, regardless of how probable the sampled action was under the current policy~\citep{williams1992simple,sutton1999policy}.
Within a single decision context, a rare negative-advantage action can disproportionately distort the update direction, even though the policy already avoids it.
Across contexts, the problem is deeper: the expected policy-gradient direction systematically over-allocates gradient budget to contexts the policy already solves well.
A classifier allocates more gradient budget to an image it already gets right 99\% of the time than to one it gets right 50\%; likewise, an easy prompt dominates a harder one simply because the model already succeeds on it.
This bias is not a finite-sample artifact: it persists even with infinite data.

We propose a simple fix: gate each gradient term with a sigmoid of \emph{delight}, the product of advantage and action surprisal, where surprisal is the negative log-probability of the sampled action under the current policy.
The resulting estimator is the \textit{Delightful Policy Gradient} (DG): one sigmoid, one multiply, and a temperature $\eta$ that we fix to $1$ throughout.
Within a single context, DG suppresses perpendicular noise from unlikely negative-advantage actions (Prop.~\ref{prop:variance}), improving directional accuracy; this variance effect vanishes as batch size grows.
Across contexts, DG shifts the expected gradient strictly closer to the supervised cross-entropy oracle (Prop.~\ref{prop:direction}); this directional effect persists even in the infinite-sample limit.

We build this case progressively across theory and experiments.
In tabular $K$-armed bandits, we isolate the two mechanisms analytically (Section~\ref{sec:bandit}).
On MNIST contextual bandits, we directly confirm both effects and show that DG closes roughly half the gap to supervised cross-entropy (Section~\ref{sec:mnist}).
On token reversal with transformers, DG's advantage compounds with difficulty, yielding a smaller scaling exponent than all baselines (Section~\ref{sec:transformer}).
On continuous control across 28 environments, DG matches or exceeds baselines without task-specific tuning (Section~\ref{sec:control}).

\section{Delightful Policy Gradient}
\label{sec:method}

We consider the standard episodic reinforcement learning setting.
At each timestep $t$, the agent observes a history $\mathcal{H}_t$, samples an action $A_t \sim \pi_\theta(\cdot \mid \mathcal{H}_t)$, and receives reward $R_t$.
Standard policy gradients form per-sample updates
$g_t = U_t \nabla_\theta \log \pi_\theta(A_t \mid \mathcal{H}_t),$
where $U_t$ denotes an advantage estimate.
Thus each score term is weighted by advantage alone, regardless of how likely the sampled action was under the current policy.
DG modulates each term by a second quantity: \emph{action surprisal}.

\subsection{Definitions}
\label{sec:definitions}
\vspace{-1mm}

The action surprisal is 
$\ell_t = -\log \pi_\theta(A_t \mid \mathcal{H}_t),$
which is large when the chosen action is unlikely under the current policy in that decision context.
This surprisal is policy-relative: it measures how unlikely the action was under the policy, not how common that action is in the environment.
It is also action-relative rather than outcome-relative: $\ell_t$ depends only on the probability the policy assigned to the sampled action, not on whether the resulting reward was good or bad.

We define \emph{delight} as
$\chi_t = U_t \ell_t,$
the product of advantage and action surprisal, which is large when an unlikely action has high advantage.%
\footnote{We write $\chi$ for the Greek \emph{chara} (delight).}
The gate is
$w_t = \sigma(\chi_t / \eta)$
for the sigmoid $\sigma(x)=\frac{1}{1+e^{-x}},$
where $\eta > 0$ controls sharpness.
DG therefore replaces the standard term $g_t$ with the gated term $w_t\, g_t$.
In other words, DG keeps the standard policy-gradient term but rescales it according to delight.

This gate weights score terms asymmetrically.
When a rare action has positive advantage ($\chi_t \gg 0$), the gate opens and $w_t \approx 1$; we call such events \emph{breakthroughs}.
When a rare action has negative advantage ($\chi_t \ll 0$), the gate closes and $w_t \approx 0$; we call such events \emph{blunders}.
For actions the policy already favors, the surprisal is small, so the gate stays near $\tfrac{1}{2}$.

The intuition for this asymmetry is simple.
A \emph{blunder}---a low-probability action with negative advantage---is already being avoided; pushing it down further has limited value.
A \emph{breakthrough}---a low-probability action with positive advantage---is a discovery the policy should exploit, and increasing its probability also creates more opportunities to learn from it in the future.
Standard policy gradients treat these two cases symmetrically; DG preserves breakthroughs while attenuating blunders.
The sigmoid gate also arises from a local entropy-regularized objective over gate values (Appendix~\ref{app:derivations}).
Figure~\ref{fig:dg_coefficient} visualizes the resulting effective coefficient on the gradient $\nabla_\theta \log \pi$.

\begin{figure}[ht!]
\centering
\begin{subfigure}[t]{0.48\columnwidth}
    \centering
    \includegraphics[width=\linewidth]{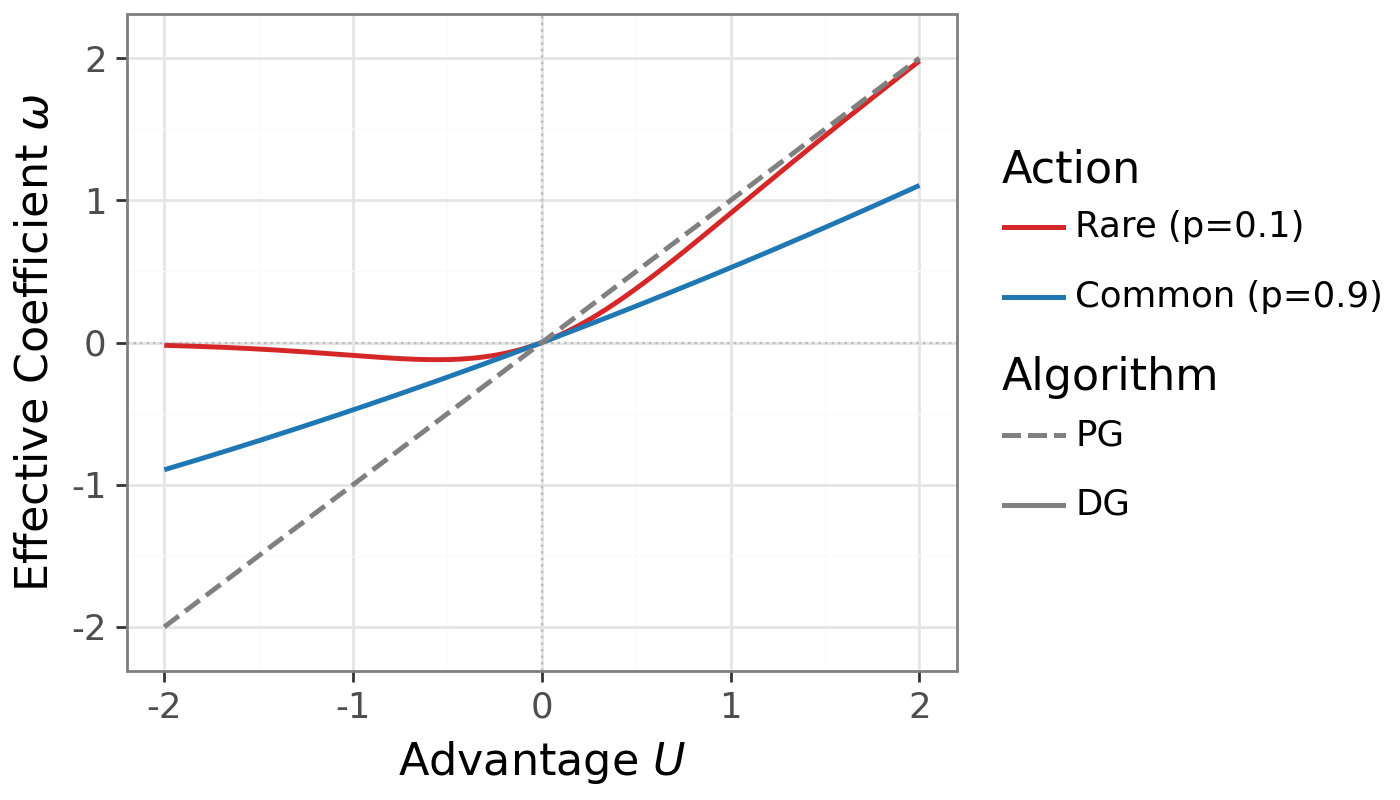}
    \vspace{-4mm}
    \caption{Coefficient $\omega$ vs.\ advantage $U$.}
    \label{fig:dg_vs_u}
\end{subfigure}
\hfill
\begin{subfigure}[t]{0.48\columnwidth}
    \centering
    \includegraphics[width=\linewidth]{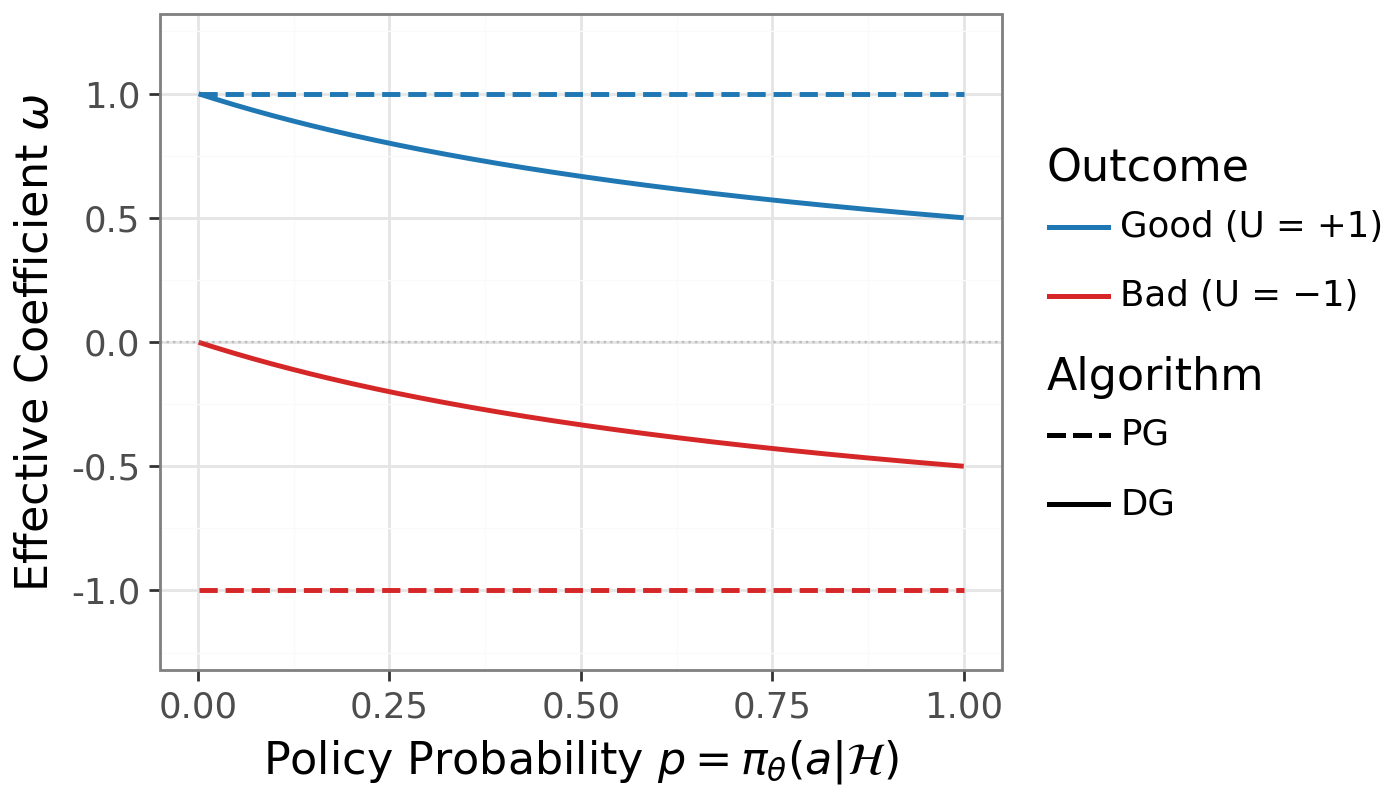}
    \vspace{-4mm}
    \caption{Coefficient $\omega$ vs.\ action probability $p$.}
    \label{fig:dg_vs_p}
\end{subfigure}
\vspace{-2mm}
\caption{Effective coefficient $\omega = w \cdot U$ weighting $\nabla_\theta \log \pi$.
DG amplifies breakthroughs (rare successes) and suppresses blunders (rare failures); PG~(dashed) is probability-blind.}
\label{fig:dg_coefficient}
\vspace{-2mm}
\end{figure}

\subsection{Estimator and Implementation}
\label{sec:estimator}

Operationally, DG simply multiplies each standard policy-gradient term
$g_t = U_t \nabla_\theta \log \pi_\theta(A_t \mid \mathcal{H}_t)$
by the gate $w_t = \sigma(\chi_t/\eta)$.
This is a drop-in replacement for standard policy gradients:
\begin{equation}
\label{eq:dg}
\Delta \theta \;\propto\; \sum_{t \in \mathcal{B}} w_t\, g_t
\;=\;
\sum_{t \in \mathcal{B}} \sigma(\chi_t / \eta)\, U_t \,\nabla_\theta \log \pi_\theta(A_t \mid \mathcal{H}_t).
\end{equation}
Because the gate reweights samples, DG changes the expected update direction, not merely its variance.
At an optimal policy, on-policy advantages vanish, so delight vanishes and optimal policies are stationary points of DG for all $\eta > 0$ (Appendix~\ref{app:derivations}).
We use $\eta = 1$ throughout and find it robust across experiments.
Algorithm~\ref{alg:dg} gives pseudocode; relative to PG, DG adds one sigmoid and one multiply per sample.

\begin{center}
\begin{minipage}{0.75\columnwidth}
\begin{algorithm}[H]
\caption{Delightful Policy Gradient (discrete actions)}
\label{alg:dg}
\begin{algorithmic}[1]
\Require Batch $\mathcal{B}$, policy $\pi_\theta$, temperature $\eta{=}1$
\State $\Delta\theta \gets 0$
\For{$t \in \mathcal{B}$}
    \State $\ell_t \gets -\log \pi_\theta(A_t \mid \mathcal{H}_t)$ \Comment{Surprisal}
    \State $\chi_t \gets U_t \cdot \ell_t$ \Comment{Delight}
    \State $w_t \gets \sigma(\chi_t / \eta)$ \Comment{Gate}
    \State $\Delta\theta \gets \Delta\theta + w_t \, U_t \, \nabla_\theta \log \pi_\theta(A_t \mid \mathcal{H}_t)$
\EndFor
\State \Return $\Delta\theta$
\end{algorithmic}
\end{algorithm}
\end{minipage}
\end{center}

For continuous actions, density-based surprisal makes $\chi = U\ell$ sensitive to action scaling.
In our control experiments, clipping log-densities to $[-10,10]$ was sufficient (Algorithm~\ref{alg:dg_continuous} in Appendix~\ref{app:control_setup}); when action scales vary substantially across dimensions, one can also whiten $\tilde{\chi}_t = (\chi_t - \mu_\chi)/(\sigma_\chi + \epsilon)$ before gating.

\section{MNIST Diagnostic}
\label{sec:mnist}

We cast MNIST classification as a one-step contextual bandit.
Given an image, the learner predicts a digit and observes only whether that prediction was correct, not the label itself.
This makes MNIST a clean test of whether policy-gradient updates recover the gradient geometry that standard supervised learning gets from full labels, without the additional complications of sequential control.

Formally, given an image $X$, the agent samples $A \in \{0,\dots,9\}$ from the policy and receives reward $R = \mathbb{I}\{A = Y\}$, where $Y$ is the true label.
The label itself is never revealed to the learner.
We train a two-layer ReLU network with Adam on batches of $B = 100$ images.

To diagnose gradient quality, we compare each batch update against two oracle directions, both computed from the true labels and therefore unavailable to the learner:
\begin{align*}
g^*_{\mathrm{PG}} &= \textstyle\sum_{x \in \mathcal{B}} p(x)\,\nabla_\theta \log \pi_\theta(Y \mid x)
&&\text{(PG oracle),}\\
g^*_{\mathrm{CE}} &= \textstyle\sum_{x \in \mathcal{B}} \nabla_\theta \log \pi_\theta(Y \mid x)
&&\text{(cross-entropy oracle),}
\end{align*}
where $p(x) := \pi_\theta(Y \mid x)$ is the probability assigned to the correct label.
The PG oracle weights each image by its current success probability, whereas the cross-entropy oracle weights all images equally.
Equivalently, these directions correspond to maximizing $\sum p(x)$ and $\sum \log p(x)$, respectively.
All results average over 100 seeds; shaded regions show $\pm$1 standard error.
Figure~\ref{fig:mnist_curve} shows that DG learns faster than PG, closing roughly half the gap to supervised cross-entropy (CE) on this configuration.
CE requires labels; PG and DG see only rewards.
The architecture and optimizer are identical across methods; the only difference is how each method weights its gradients.

\begin{figure}[ht!]
\centering
\vspace{-2mm}
\includegraphics[width=0.5\columnwidth]{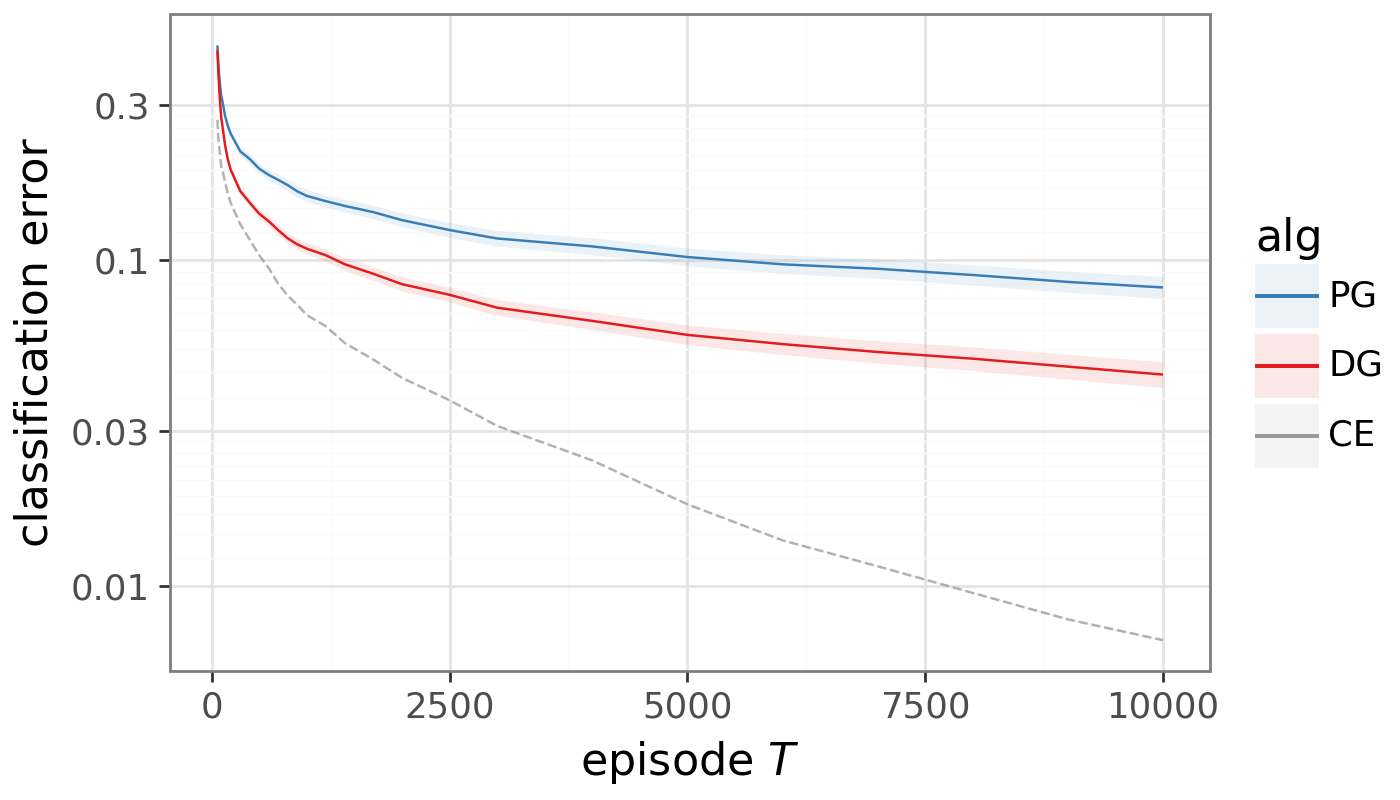}
\vspace{-3mm}
\caption{MNIST classification error.
Supervised CE requires labels; PG and DG do not.}
\vspace{-2mm}
\label{fig:mnist_curve}
\end{figure}

To test whether this advantage is merely variance reduction, Figure~\ref{fig:mnist_limit} varies both the baseline and the number of independent action samples drawn per image, denoted $S$.
The advantage takes the form $U = R - b$, and we consider four baselines: $b = 0$ (zero), $b = 0.5$ (constant), $b = \hat{\mathbb{E}}[R \mid x]$ using the agent's own probability estimate (expected), and $b = \mathbb{E}[R \mid x]$ using the true label probability (oracle).
As $S$ increases, PG converges from above to the error floor set by the exact PG oracle $g^*_{\mathrm{PG}}$ (dashed line in Figure~\ref{fig:mnist_limit}).
DG surpasses this floor for every baseline: with the expected baseline, DG at $S = 1$ already matches the level that PG approaches only as $S \to \infty$.
Because this gain persists at large $S$, it cannot be explained by variance reduction alone: DG changes the expected gradient direction itself.

\begin{figure}[ht!]
\centering
\vspace{-1mm}
\includegraphics[width=\columnwidth]{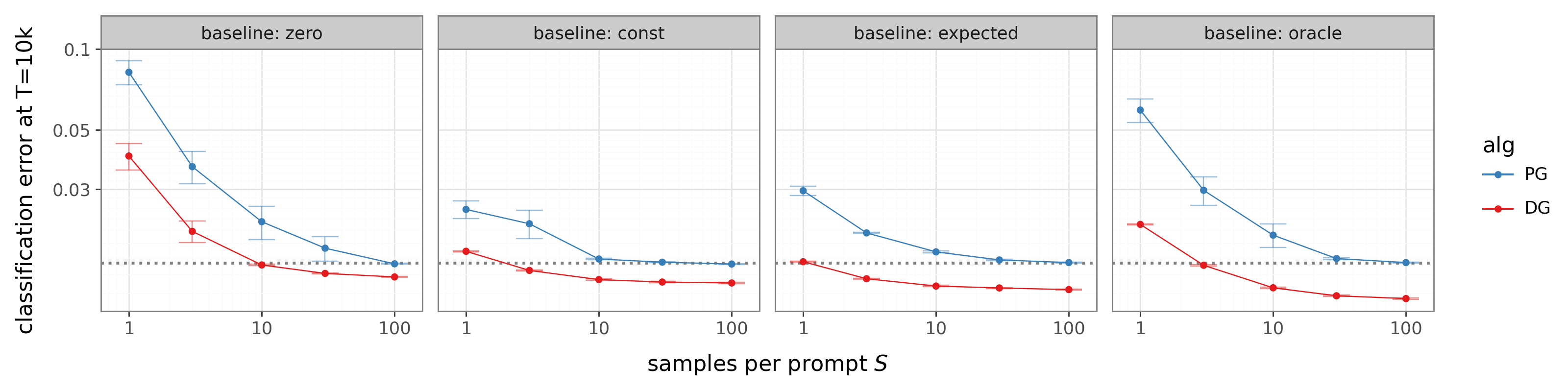}
\vspace{-3mm}
\caption{Classification error at $T = 10\text{k}$ vs.\ samples per image $S$, faceted by baseline.
Dashed line: error from the exact PG-oracle gradient $g^*_{\mathrm{PG}}$, PG's best achievable direction.}
\vspace{-1mm}
\label{fig:mnist_limit}
\end{figure}

Figure~\ref{fig:mnist_grad} separates the two mechanisms by measuring alignment to both oracles at $S = 1$ and $S = 100$.
Against $g^*_{\mathrm{PG}}$~(a), DG's advantage shrinks at $S = 100$: this is the within-context variance effect disappearing as sampling noise is reduced.
Against $g^*_{\mathrm{CE}}$~(b), DG's advantage persists at $S = 100$: this is the cross-context directional effect, which remains even when sampling noise is negligible.
The only difference between the two oracles is how they weight images: the PG oracle allocates more budget to images the model already classifies correctly, whereas the cross-entropy oracle treats every image equally.
DG compresses these weights toward equality, reallocating gradient budget from well-solved images to harder ones.
Section~\ref{sec:bandit} formalizes both mechanisms in settings where the true gradients are analytically tractable.

\begin{figure}[ht!]
\centering
\vspace{-2mm}
\begin{subfigure}[t]{0.48\columnwidth}
    \centering
    \includegraphics[width=\linewidth]{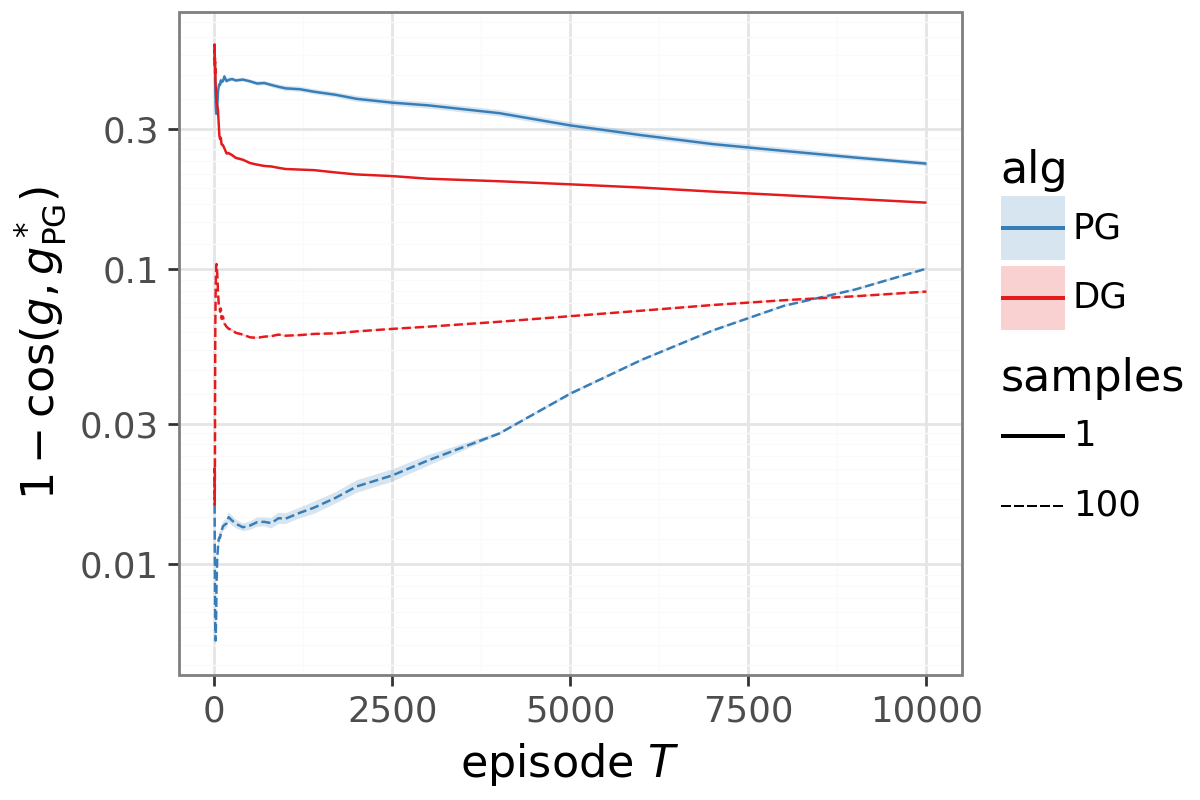}
    \vspace{-2mm}
    \caption{Misalignment to $g^*_{\mathrm{PG}}$.}
    \label{fig:mnist_grad_pg}
\end{subfigure}
\hfill
\begin{subfigure}[t]{0.48\columnwidth}
    \centering
    \includegraphics[width=\linewidth]{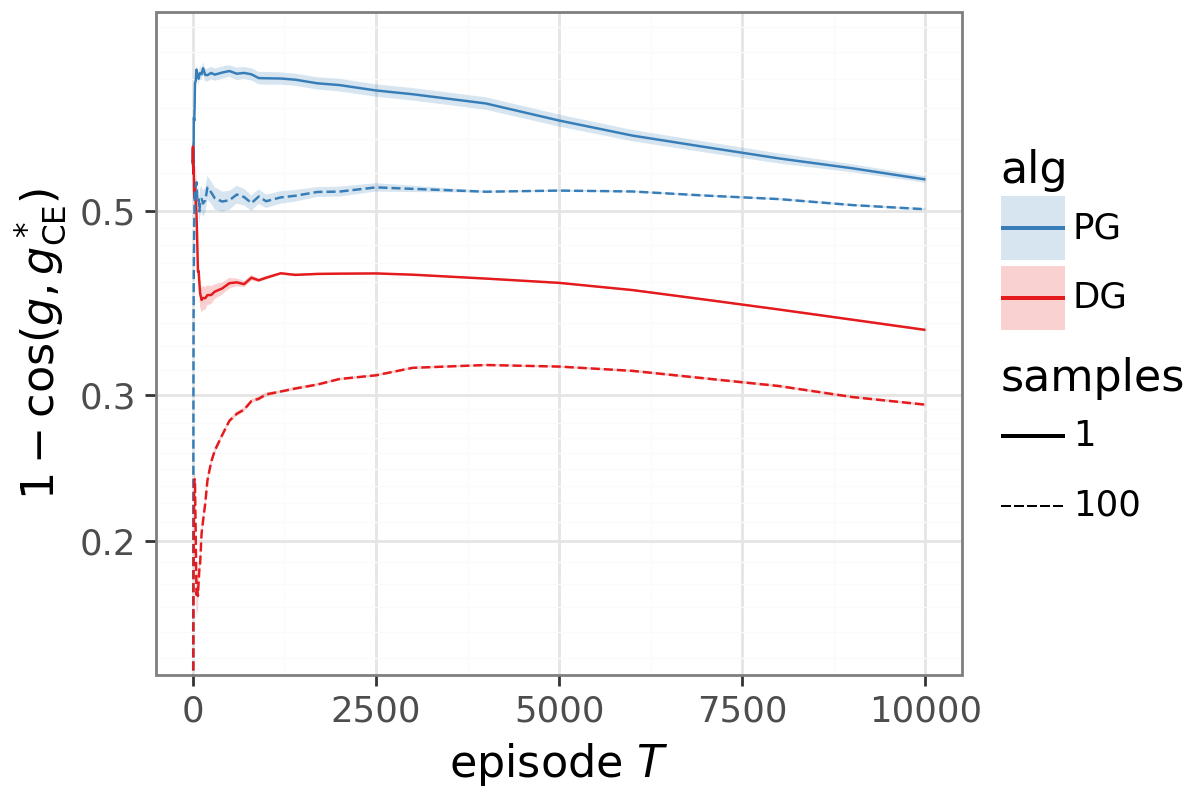}
    \vspace{-2mm}
    \caption{Misalignment to $g^*_{\mathrm{CE}}$.}
    \label{fig:mnist_grad_ce}
\end{subfigure}
\vspace{-2mm}
\caption{Gradient misalignment over training at $S = 1$ (solid) and $S = 100$ (dashed).
\textbf{(a)}~DG's advantage diminishes with $S$ (variance).
\textbf{(b)}~DG's advantage persists at $S = 100$ (directional).}
\label{fig:mnist_grad}
\vspace{-2mm}
\end{figure}

Appendix~\ref{app:mnist} provides extensive ablations showing that DG's gains persist across learning rates, batch sizes, network widths, and baselines.
Validation error tracks training error closely, indicating no overfitting.
The multiplicative form $\chi = U \cdot \ell$ also outperforms additive alternatives and entropy regularization.
Taken together, these results suggest that DG addresses a core policy-gradient mismatch that appears even in canonical classification under bandit feedback.

\section{Tabular Analysis}
\label{sec:bandit}

The MNIST diagnostics revealed two effects: DG denoises the gradient estimate (Figure~\ref{fig:mnist_grad_pg}) and rotates its expected direction toward the cross-entropy oracle (Figure~\ref{fig:mnist_grad_ce}).
To isolate these mechanisms analytically, we move to the tabular setting, replacing the neural network with explicit tables over actions, and study $K$-armed bandits under simple assumptions: symmetry, orthogonality, and normalized steps.

The MNIST, transformer, and control experiments violate all of these assumptions; the tabular analysis isolates mechanism rather than modeling the full applications.
We parameterize policies by logits $z$, so $\phi(a) = \nabla_z \log \pi(a) = e_a - \pi$.%
\footnote{Elsewhere $\phi$ denotes the score in general $\theta$-space; under logits, it simplifies to $e_a - \pi$.}
To model direction-limited optimization, each step takes the form $z \leftarrow z + \alpha g / \|g\|$: every update has norm $\alpha$, so better direction is the only way to learn faster.
This is a useful idealization of large-scale training regimes in which adaptive optimization and gradient clipping can attenuate differences in raw gradient norm, making update direction increasingly important for progress (Appendix~\ref{app:cosine_progress}).

\subsection{Single context: variance reduction}
\label{subsec:variance}

We begin with a single context: a $K$-armed bandit with one correct action $y^*$ among $K \ge 3$ choices, like classifying a single MNIST image but with the neural network replaced by an explicit probability table.
The reward is $R = \mathbb{I}\{A = y^*\}$.
In any single context of this form, the reward and cross-entropy oracles coincide.
The score identity $\sum_a \pi(a)\,\phi(a) = 0$ implies $\mathbb{E}[g_{\mathrm{PG}}] = \pi(y^*)\,\phi(y^*) \propto g^*_{\mathrm{CE}}$.
The only way to improve gradient quality is therefore to reduce perpendicular sampling noise.
We write $\Pi_\perp$ for projection orthogonal to $\nabla J$ and $\mathrm{Var}_\perp(g) = \mathbb{E}\|\Pi_\perp(g)\|^2$ for the perpendicular variance.

Symmetry yields closed-form equalities.
We specialize to a policy with $\pi(y^*) = 1 - \varepsilon$ and $\pi(a) = \varepsilon/(K{-}1)$ for each incorrect action, with baseline $b \in (0,1)$.
The DG gate then takes a common value $w_+$ on the correct action and $w_-$ on all incorrect actions, with $w_- < \tfrac{1}{2} < w_+$.%
\footnote{Explicitly: $w_+ = \sigma\!\big((1{-}b)\log\tfrac{1}{1-\varepsilon}/\eta\big)$ and $w_- = \sigma\!\big(-b\log\tfrac{K-1}{\varepsilon}/\eta\big)$.}

\begin{proposition}[Variance reduction in symmetric bandits]
\label{prop:variance}
In the bandit above, for any $\eta > 0$:

\textbf{(i)} DG preserves the expected gradient direction: $\mathbb{E}[g_{\mathrm{DG}}] = s \cdot g^*_{\mathrm{PG}}$, where $s = (1{-}b)\,w_+ + b\,w_- > 0$.

\textbf{(ii)} DG reduces perpendicular variance by exactly $w_-^2$: $\mathrm{Var}_\perp(g_{\mathrm{DG}}) = w_-^2 \cdot \mathrm{Var}_\perp(g_{\mathrm{PG}})$.

\textbf{(iii)} For the batch mean $\bar{g} = \frac{1}{B}\sum_{i=1}^B g_i$, in the regime where $\bar{g}$ concentrates around its mean,
\begin{equation}
\label{eq:gap_ratio}
\frac{1 - \mathbb{E}[\cos(\bar{g}_{\mathrm{DG}},\, g^*_{\mathrm{PG}})]}
     {1 - \mathbb{E}[\cos(\bar{g}_{\mathrm{PG}},\, g^*_{\mathrm{PG}})]}
\;\approx\; \frac{w_-^2}{s^2} \;<\; 1.
\end{equation}
DG always reduces the alignment gap; both methods converge to cosine~$1$ as $B \to \infty$.
\end{proposition}

The ratio $w_-^2/s^2$ measures the fraction of PG's cosine gap that DG retains.
Since $\sigma(-x) \le e^{-x}$, the gate satisfies $w_- \le \sqrt{\varepsilon/(K{-}1)}$ for $b = \tfrac{1}{2}$ and $\eta = 1$, giving $w_-^2/s^2 \le 16\,\varepsilon/(K{-}1)$.
Evaluating the exact ratio confirms that the reduction is substantial: for $K{=}100$ at $\varepsilon{=}0.5$, DG retains only $4\%$ of PG's cosine gap; at $\varepsilon{=}0.1$, only $1\%$.

\paragraph{Beyond symmetry.}
Without uniform incorrect-action probabilities, each incorrect action receives a different gate, so $\mathbb{E}[g_{\mathrm{DG}}]$ is no longer exactly collinear with $\nabla J$.
However, tail suppression still holds: $w_-(a) \le \pi(a)^{b/\eta}$, so rare actions are damped more aggressively than under PG.
As the policy concentrates ($\varepsilon \to 0$), the directional bias vanishes and the variance bound tightens (Appendix~\ref{app:nonsymmetric}).

We validate this picture with $K{=}100$, $B{=}100$, $\alpha{=}0.1$, and $\eta{=}1$, averaged over 100 seeds.
Despite identical step magnitudes, DG converges faster (Figure~\ref{fig:bandit}a) and maintains lower misalignment throughout training (Figure~\ref{fig:bandit}b).
Late in training, incorrect actions become rare but each delivers a large perpendicular kick; PG's misalignment rebounds while DG's stays suppressed.

\begin{figure}[ht!]
\centering
\vspace{-1mm}
\begin{subfigure}[t]{0.48\columnwidth}
    \centering
    \includegraphics[width=\linewidth]{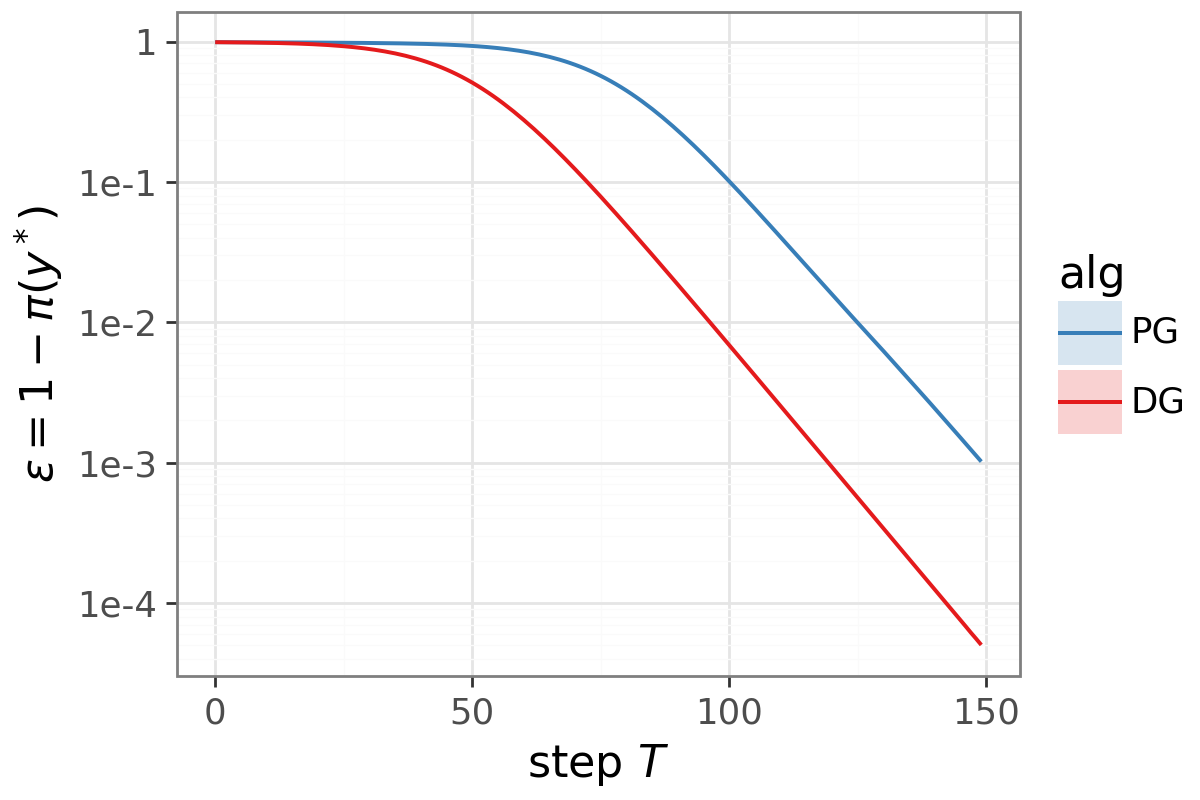}
    \vspace{-4mm}
    \caption{Error $\varepsilon = 1 - \pi(y^*)$.}
    \label{fig:bandit_conv}
\end{subfigure}
\hfill
\begin{subfigure}[t]{0.48\columnwidth}
    \centering
    \includegraphics[width=\linewidth]{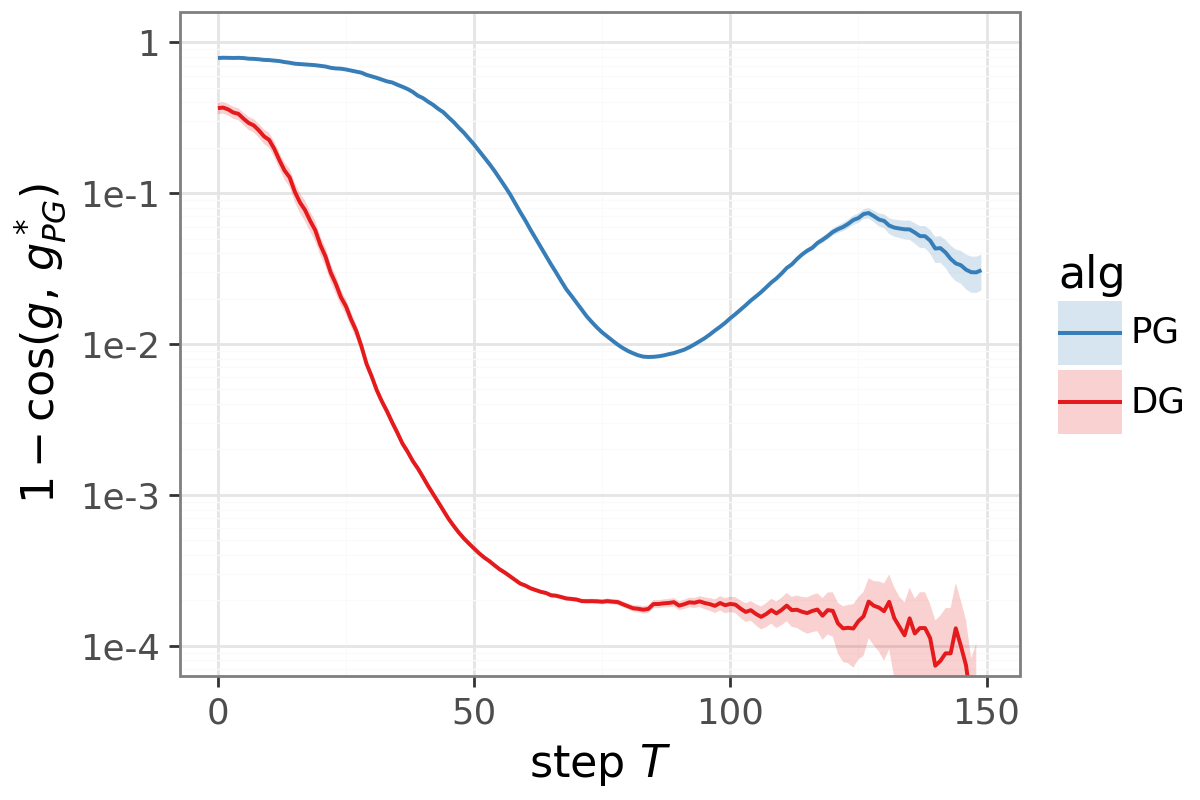}
    \vspace{-4mm}
    \caption{Misalignment $1 - \cos(g,\, g^*_{\mathrm{PG}})$.}
    \label{fig:bandit_cos}
\end{subfigure}
\vspace{-2mm}
\caption{Symmetric bandit, normalized steps ($K{=}100$, $B{=}100$, $\alpha{=}0.1$).}
\vspace{-1mm}
\label{fig:bandit}
\end{figure}

\subsection{Multiple contexts: directional improvement}
\label{subsec:direction}

For a single context, the reward and cross-entropy oracles coincide (Section~\ref{subsec:variance}), so DG can only reduce variance.
In practice, however, a single gradient step must improve many contexts at once: each MNIST image is a separate $K{=}10$ classification problem, and the update must allocate learning across all of them.
With $N$ independent contexts at each step (a \emph{contextual bandit}), the two oracles diverge.
PG allocates gradient budget proportional to $p_n$, the probability of the correct action in context $n$; CE allocates equally, making direction itself a degree of freedom.
MNIST confirmed this: DG surpasses PG's performance floor even at large $S$ (Figure~\ref{fig:mnist_limit}).

Formally, suppose the agent faces $N$ independent contexts at each step, sums their gradients, and takes a single normalized step.
Each context $n$ contributes an orthogonal gradient component $v_n$; let $p_n := \pi_n(y_n)$ denote the probability of the correct action.
With baseline $b = 0$, only correct actions contribute, and the three gradient directions differ only in how they weight contexts:%
\footnote{With baseline $b > 0$, DG additionally suppresses incorrect-action variance (Prop.~\ref{prop:variance}). Setting $b = 0$ isolates the cross-context reweighting for this subsection.}
\begin{align}
\label{eq:three_directions}
g^*_{\mathrm{CE}} &= \textstyle\sum_n v_n  &&\text{(cross-entropy: equal weight),} \nonumber\\[2pt]
g^*_{\mathrm{PG}} &= \textstyle\sum_n p_n \, v_n  &&\text{(PG: weight $\propto p_n$),}\\[2pt]
\mathbb{E}[g_{\mathrm{DG}}] &= \textstyle\sum_n p_n \, \sigma\!\big(({-}\log p_n)/\eta\big) \, v_n  &&\text{(DG: weights compressed by gate).} \nonumber
\end{align}
PG weights each context by $p_n$, so well-solved contexts dominate the gradient.
The gate $\sigma((-\log p_n)/\eta)$ is larger when $p_n$ is small and smaller when $p_n$ is large, partially cancelling the $p_n$ weighting and redistributing budget toward hard contexts (Figure~\ref{fig:mnist_grad_ce}).

These three directions are not arbitrary.
Under normalized steps, $c_n \propto p_n$ is greedy-optimal for $\Delta(\sum p_n)$ and $c_n \propto 1$ for $\Delta(\sum \log p_n)$ (Lemma~\ref{lem:greedy}).
PG's direction is myopically optimal but self-reinforcing: since $\nabla p_n = p_n v_n$, improving an easy context makes its gradient larger, attracting even more budget next step.
DG moves weights toward $c_n \propto 1$, rebalancing budget to hard contexts where each step yields more long-run progress.
As with the Kelly criterion~\citep{kelly1956new}, greedy single-step optimality does not imply optimal long-run compounding.

\begin{lemma}[Greedy directions under normalized steps]
\label{lem:greedy}
Under a normalized step with $g = \sum_n c_n v_n$:
the maximizer of $\Delta(\sum p_n)$ is $c_n \propto p_n$, and the maximizer of $\Delta(\sum \log p_n)$ is $c_n \propto 1$.
\end{lemma}

\begin{proposition}[Directional improvement toward cross-entropy]
\label{prop:direction}
Let $N = 2$ with $p_1 \neq p_2$ and $\eta > 1/2$.
Then 
$
\cos\!\big(\mathbb{E}[g_{\mathrm{DG}}],\; g^*_{\mathrm{CE}}\big)
\;>\;
\cos\!\big(g^*_{\mathrm{PG}},\; g^*_{\mathrm{CE}}\big).
$
DG's expected direction is strictly closer to the cross-entropy oracle than PG's, for any batch size, including the limit $B \to \infty$.
\end{proposition}
\textit{Proof sketch.}
The cosine between $c_1 v_1 + c_2 v_2$ and $v_1 + v_2$ is maximized at ratio $r = c_1/c_2 = 1$ (Appendix~\ref{app:prop_b}).
The DG weight function $h(p) = p\,\sigma((-\log p)/\eta)$ is increasing for $\eta > 1/2$, and the sigmoid factor is decreasing in $p$, so DG compresses PG's ratio: $1 < r_{\mathrm{DG}} < r_{\mathrm{PG}}$, achieving higher cosine.
Appendix~\ref{app:prop_b_general} extends to arbitrary $N$ via a path argument; Figure~\ref{fig:multi_context} confirms the compression at $N{=}100$.

We validate with $N{=}100$ independent contexts, $K{=}10$ actions each, $\mathcal{N}(0,1)$ logit init, $\alpha{=}0.1$, averaged over 100 seeds.
At each step we compute \emph{exact} population gradients (no sampling noise), so PG follows $g^*_{\mathrm{PG}}$ exactly.
Despite this, DG converges faster (Figure~\ref{fig:multi_context}a): its rebalancing toward hard contexts compounds over many steps.
The cosine gap (Figure~\ref{fig:multi_cos}) persists throughout training, confirming a directional effect rather than variance reduction.

\begin{figure}[ht]
\centering
\begin{subfigure}[t]{0.48\columnwidth}
    \centering
    \vspace{-1mm}
    \includegraphics[width=\linewidth]{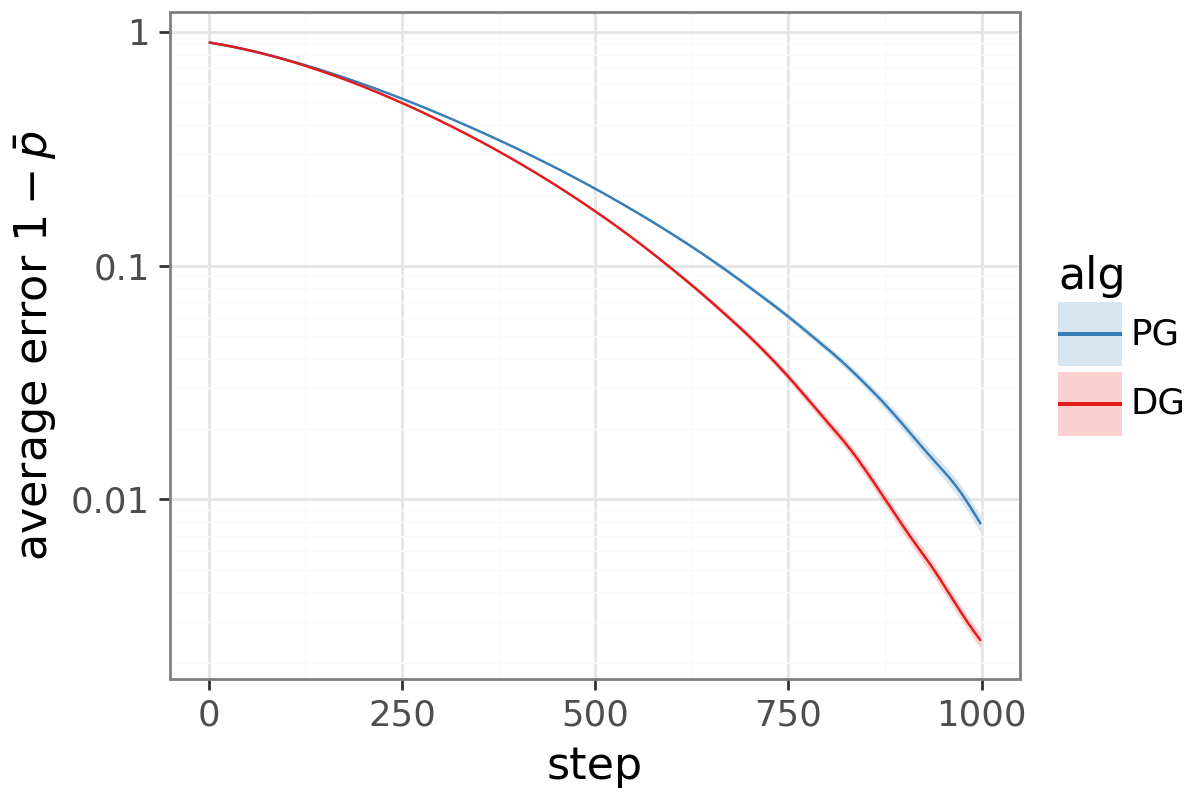}
    \vspace{-5mm}
    \caption{Average error $1 - \bar{p}$.}
    \label{fig:multi_reward}
\end{subfigure}
\hfill
\begin{subfigure}[t]{0.48\columnwidth}
    \centering
    \vspace{-1mm}
    \includegraphics[width=\linewidth]{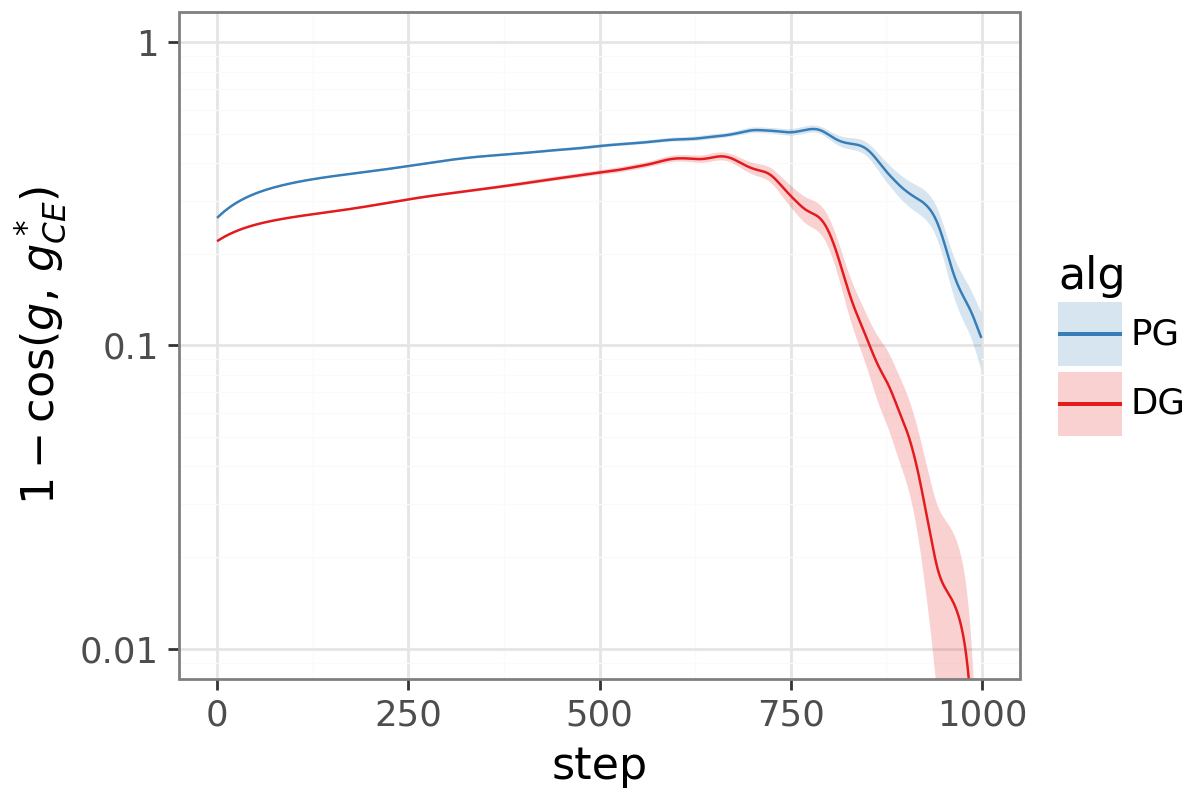}
    \vspace{-5mm}
    \caption{Misalignment to $g^*_{\mathrm{CE}}$.}
    \label{fig:multi_cos}
\end{subfigure}
\vspace{-1mm}
\caption{$N{=}100$ independent contexts, $K{=}10$ actions, $\mathcal{N}(0,1)$ init, exact gradients.
DG converges faster~(a) by rebalancing gradient budget to hard contexts~(b).}
\vspace{-1mm}
\label{fig:multi_context}
\end{figure}

The two mechanisms map onto a bias--variance decomposition.
For a single context, $g^*_{\mathrm{PG}} = g^*_{\mathrm{CE}}$, so DG's advantage is pure variance reduction and vanishes as $B \to \infty$ (Prop.~\ref{prop:variance}).
Across multiple contexts, Proposition~\ref{prop:direction} shows something stronger: DG improves the expected update itself, not just its noisy estimate.
The induced bias is beneficial, rotating the gradient toward cross-entropy and rebalancing learning toward hard contexts.
This means DG's advantage is not a finite-sample artifact: even with exact gradients, standard policy gradients can allocate update budget poorly, and DG corrects that defect.

\section{Transformer Sequence Modeling}
\label{sec:transformer}

The MNIST and bandit experiments isolated DG's effect on update geometry in single-step problems.
We now test the same mechanism in a sequential setting with memory, autoregressive generation, and temporal credit assignment.

We introduce the \textsc{Token Reversal} task (Figure~\ref{fig:binary_reverse}).
An input sequence $x_{1:H}$ of length $H$ is drawn uniformly from a vocabulary of size $M$, and a decoder-only Transformer autoregressively predicts the reverse sequence, $\hat{y}_t = x_{H-t+1}$.
This is a controlled model of token-space reasoning: the learner must attend to the input, retain sequence structure, and generate a coherent output autoregressively.
Because the output space has size $M^H$, fully correct sequences become exponentially rarer as horizon and vocabulary grow.
We report \emph{sequence error}, the fraction of output tokens predicted incorrectly; full details are in Appendix~\ref{app:scaling}.

We compare DG against REINFORCE, PPO~\citep{schulman2017proximal}, and PMPO~\citep{abdolmaleki2024preference}, each tuned over its key hyperparameters (Appendix~\ref{app:scaling_tuning}).
All methods share the same Transformer architecture, optimizer, and compute budget; results average over 30 seeds.

On a default configuration ($H{=}10$, $M{=}2$, $K{=}1{,}000$ episodes), Figure~\ref{fig:token_regret} shows that DG converges faster and reaches lower sequence error than all baselines.
Multiplicative advantage $\times$ surprisal gating provides a qualitatively different signal from additive entropy bonuses or trust-region constraints.
UCB-style mixtures $(1{-}\alpha)U + \alpha\ell$ improve over REINFORCE but do not match DG (Appendix~\ref{app:scaling_alternative}).

\noindent
\begin{minipage}[t]{0.48\columnwidth}
  \centering
  \vspace{0pt}
  \vspace{-1mm}
  \includegraphics[width=\linewidth]{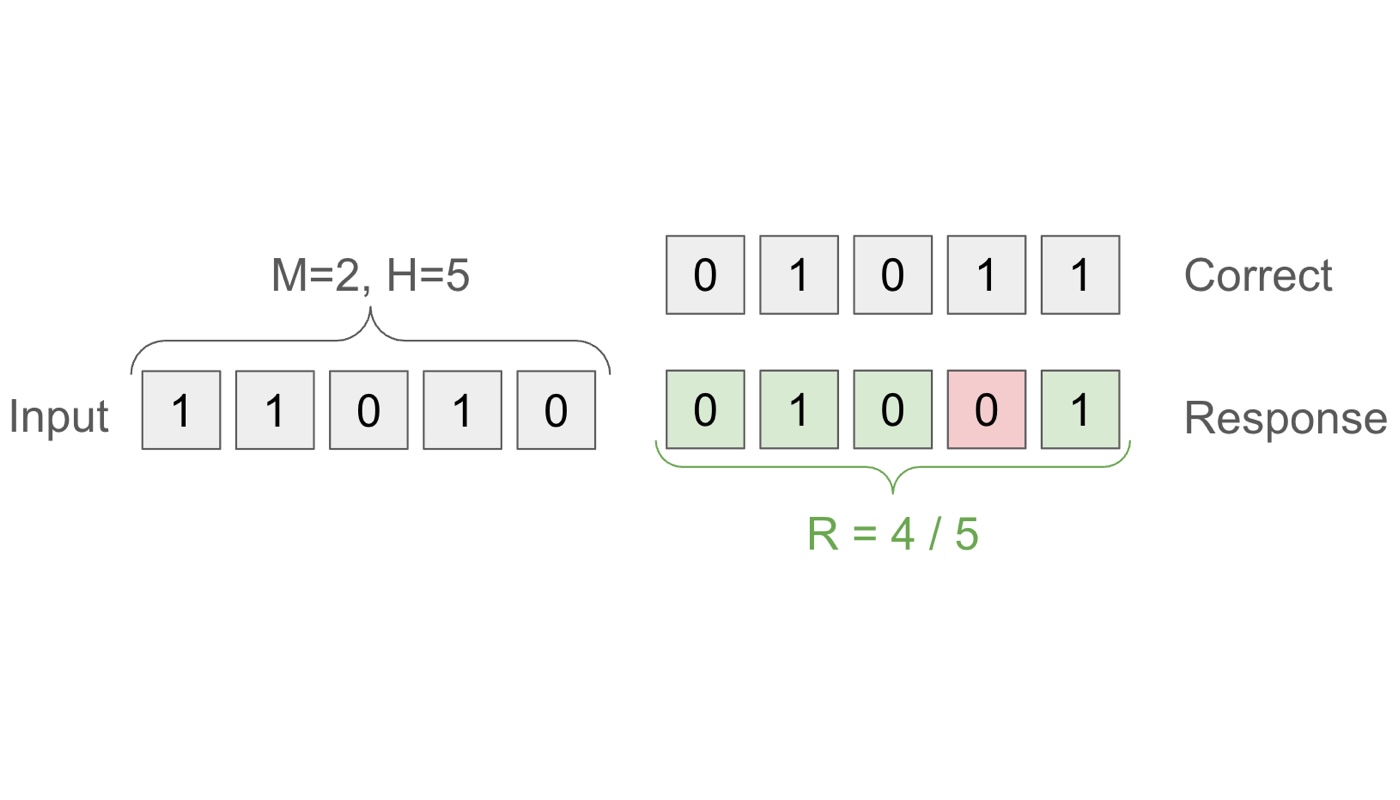}
  \vspace{-5mm}
  \captionof{figure}{Token Reversal task.}
  \vspace{-1mm}
  \label{fig:binary_reverse}
\end{minipage}
\hfill
\begin{minipage}[t]{0.48\columnwidth}
  \centering
  \vspace{0pt}
  \vspace{-1mm}
  \includegraphics[width=\linewidth]{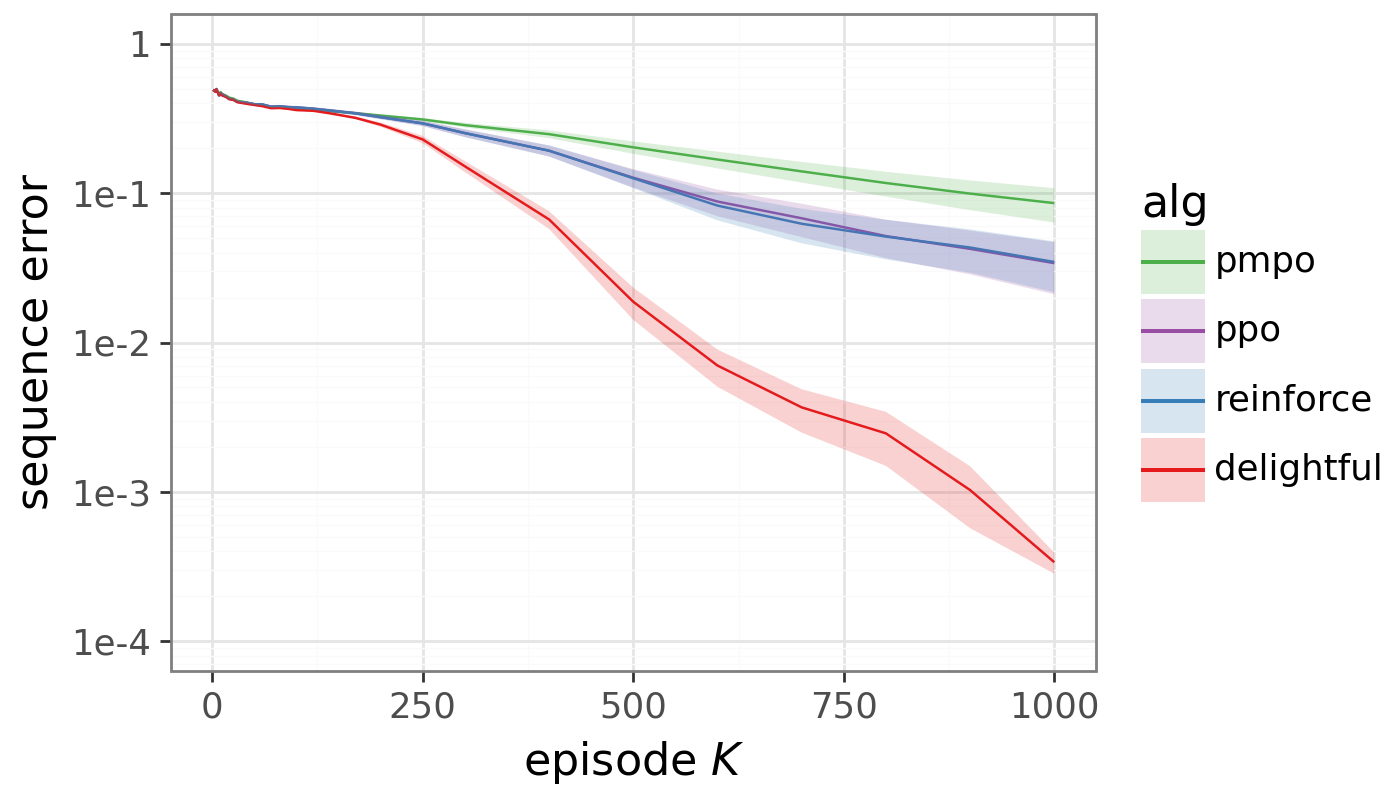}
  \vspace{-5mm}
  \captionof{figure}{Sequence error on Token Reversal.}
  \vspace{-1mm}
  \label{fig:token_regret}
\end{minipage}

To test whether DG's advantage grows with difficulty, we increase $H$ and $M$ while holding the training budget fixed at $K{=}10\text{k}$ episodes.
Figure~\ref{fig:scaling_laws}a--b shows final sequence error; Figure~\ref{fig:scaling_laws}c--d shows cumulative error on log-log axes.
A sharp contrast emerges: baseline performance deteriorates rapidly beyond a complexity threshold, whereas DG degrades much more gracefully.
The log-log plots reveal approximate power-law scaling, with DG achieving a smaller exponent; its advantage compounds with difficulty.

\begin{figure*}[ht]
\centering
\begin{subfigure}[b]{0.24\textwidth}
    \includegraphics[width=\linewidth]{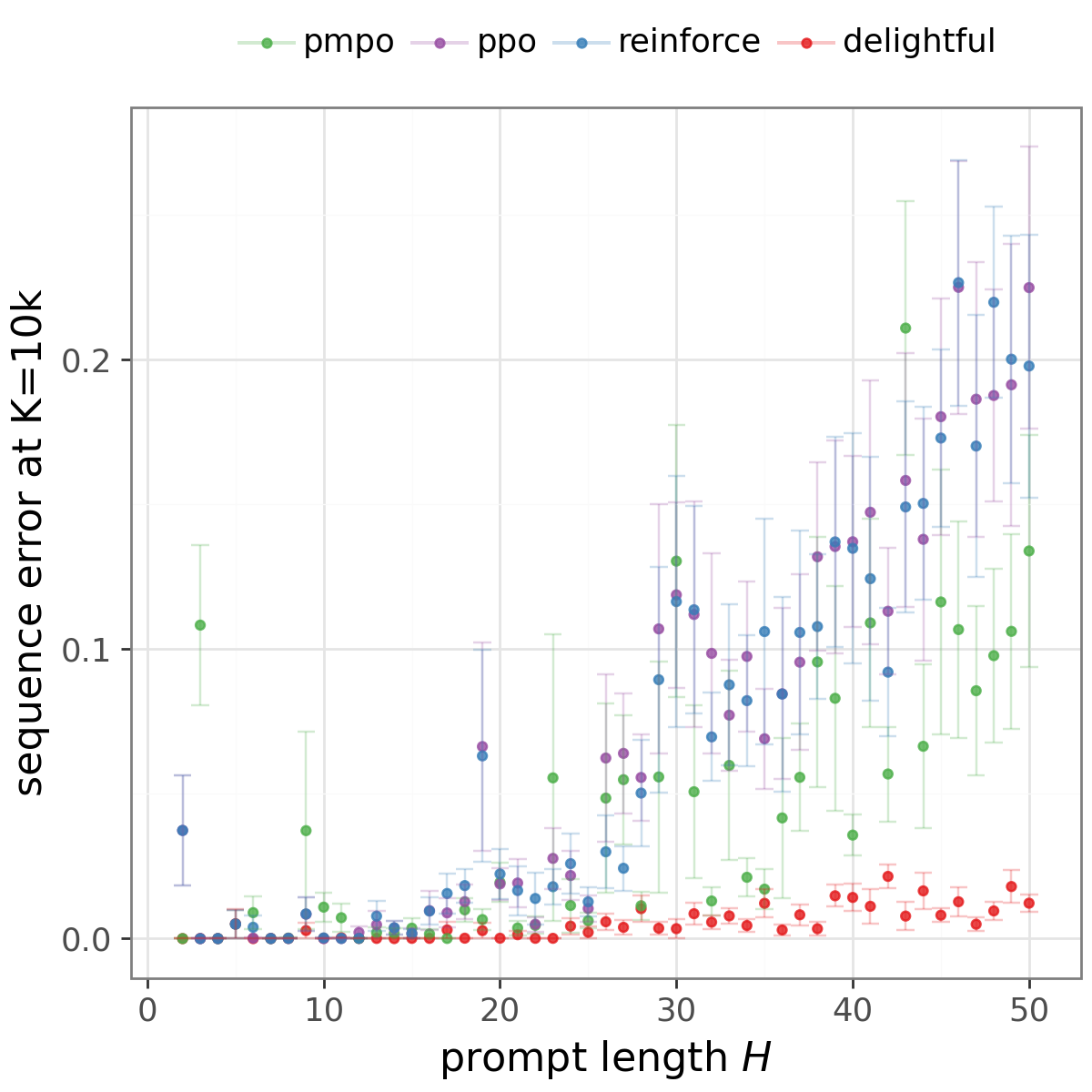}
    \caption{Final error vs.\ $H$}
\end{subfigure}
\hfill
\begin{subfigure}[b]{0.24\textwidth}
    \includegraphics[width=\linewidth]{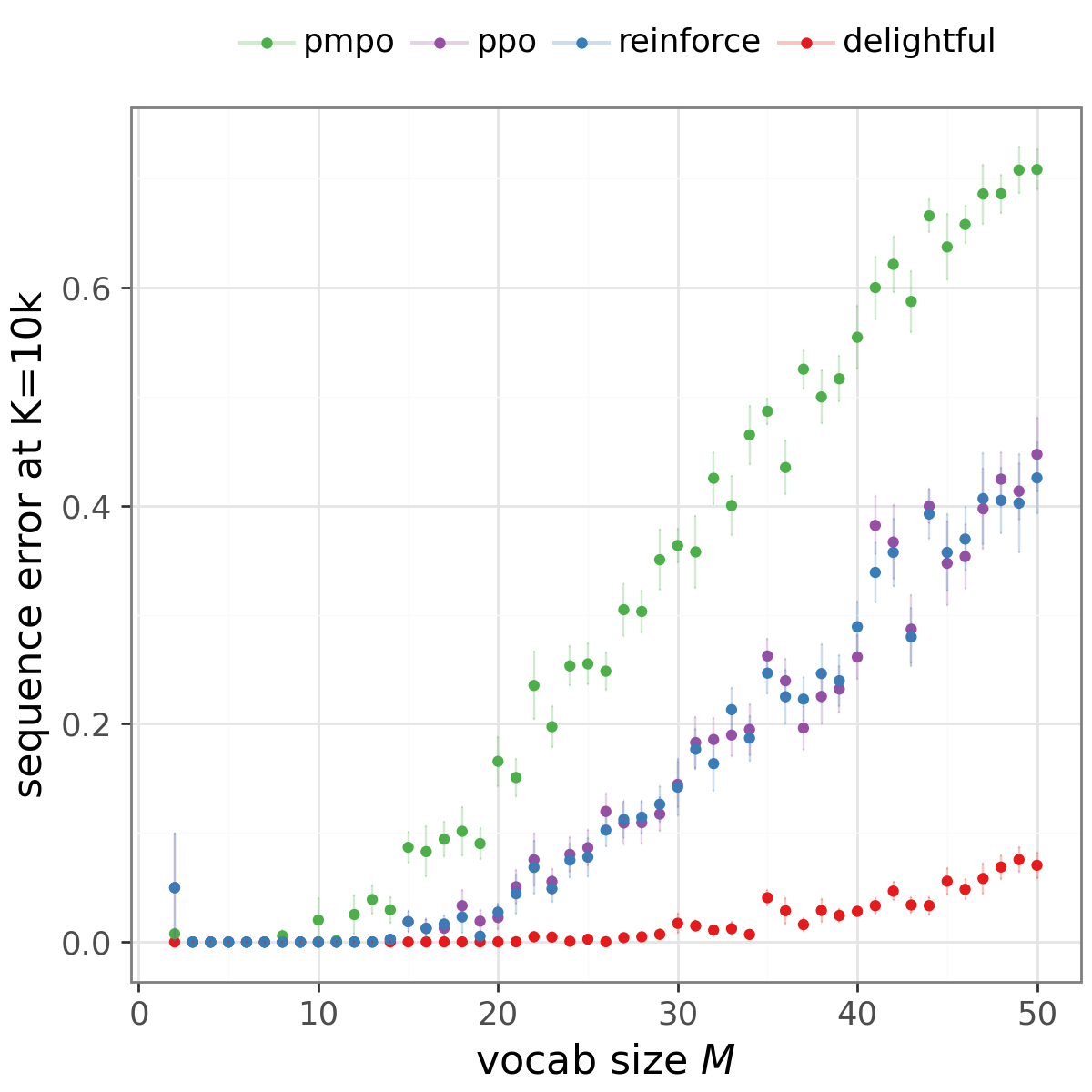}
    \caption{Final error vs.\ $M$}
\end{subfigure}
\hfill
\begin{subfigure}[b]{0.24\textwidth}
    \includegraphics[width=\linewidth]{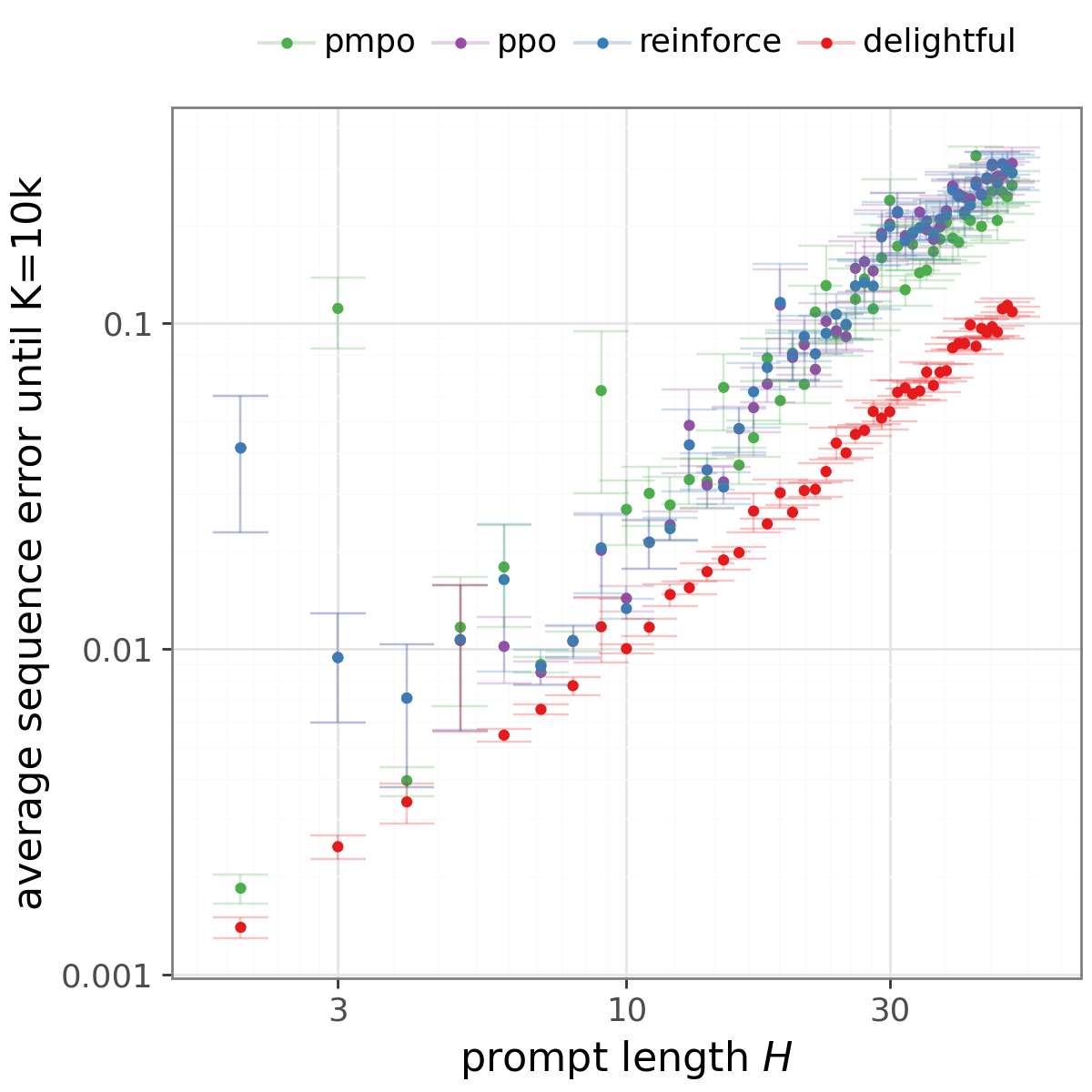}
    \caption{Avg error vs.\ $H$}
\end{subfigure}
\hfill
\begin{subfigure}[b]{0.24\textwidth}
    \includegraphics[width=\linewidth]{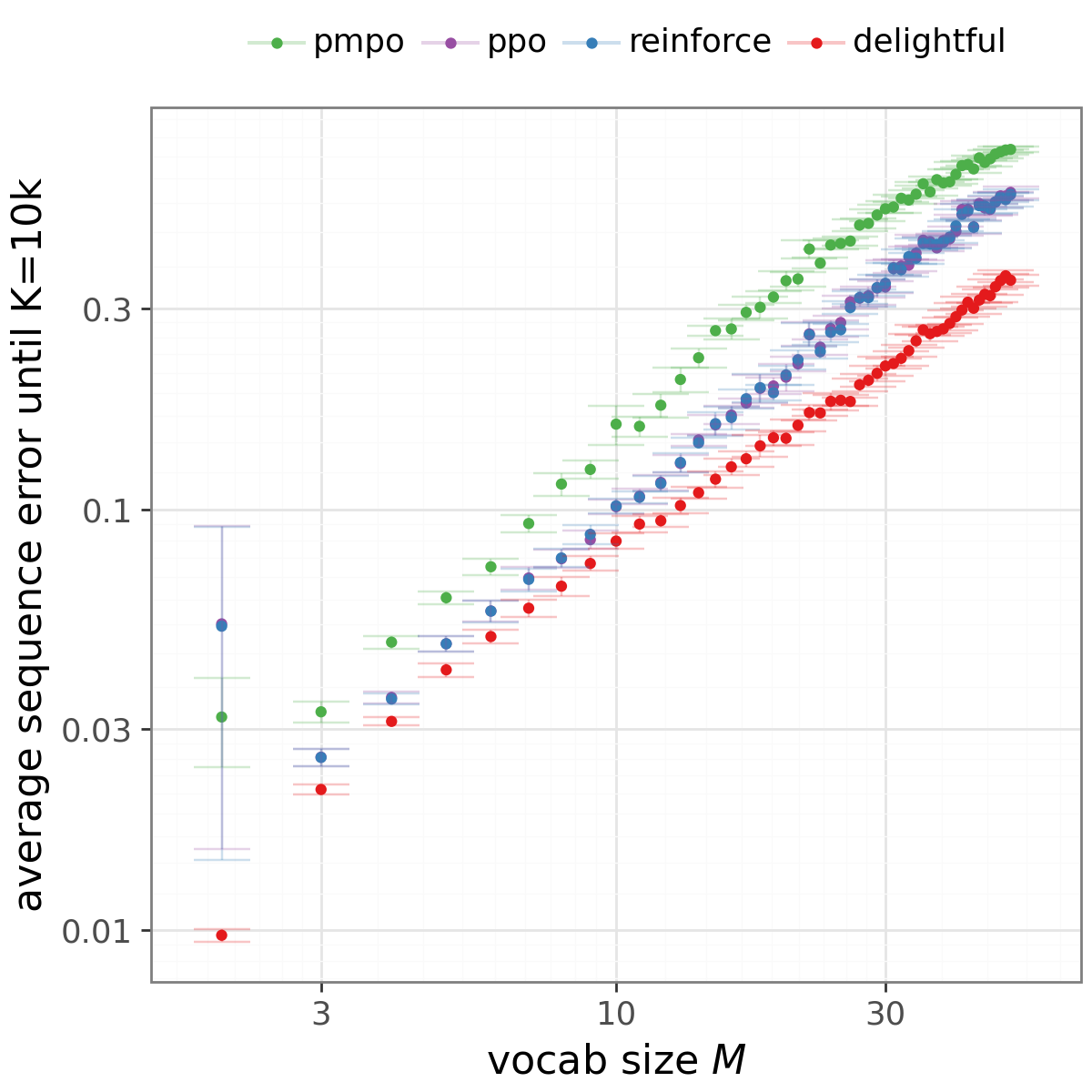}
    \caption{Avg error vs.\ $M$}
\end{subfigure}
\caption{Scaling with task complexity.
\textbf{(a,b)}~Final sequence error after $K{=}10\text{k}$ episodes; baselines degrade beyond a threshold while DG degrades more gracefully.
\textbf{(c,d)}~Log-log cumulative error reveals approximate power-law scaling; DG achieves a smaller exponent.}
\label{fig:scaling_laws}
\end{figure*}

The scaling advantage reflects the mechanism identified in the bandit analysis: as $M^H$ grows, reallocating gradient weight from high-surprisal failures to high-surprisal successes becomes increasingly valuable.
Appendix~\ref{app:scaling_environment} tests eight task variants; DG leads in every configuration, with larger margins in harder settings.
If these trends transfer to larger-scale sequence generation, where long horizons and large vocabularies make correct outputs increasingly rare, the advantage of rebalancing gradient weight could be substantial.

\section{Continuous Control}
\label{sec:control}

Token Reversal showed that DG's advantage grows with task complexity in a discrete sequential setting.
We now test whether the same mechanism transfers to continuous control, where the policy defines an action density rather than a discrete distribution.

We evaluate on the DeepMind Control Suite~\citep{tassa2018deepmind}: 28 environments, 10 million steps, and 3 seeds per environment.
All methods share the same actor--critic architecture, critic algorithm (Retrace~\citep{munos2016safe} with replay), and optimizer; only the policy update rule differs.
DG is applied only to the actor, reweighting on-policy score terms without importance ratios.
For continuous actions, delight uses clipped log-densities (Section~\ref{sec:estimator}).

DG matches or exceeds the best baseline on a majority of environments and avoids catastrophic failures: PPO collapses on \texttt{humanoid:run}, while REINFORCE fails on \texttt{hopper:hop}.
On \texttt{humanoid:run}, DG discovers a successful gait while all baselines plateau (Figure~\ref{fig:reward_humanoid_run}).
Across all 84 runs, DG is never the worst method (Figure~\ref{fig:reward_complete}) and achieves the lowest average regret throughout training (Figure~\ref{fig:regret_ave_control}).
It also remains competitive with highly tuned MPO~\citep{abdolmaleki2018maximum} and SAC~\citep{haarnoja2018soft} despite using no task-specific tuning (Appendix~\ref{app:control_sota}).
These results show that delight-based reweighting remains effective with continuous action densities.

\begin{figure}[ht!]
\centering
\vspace{-1mm}
\includegraphics[width=\columnwidth]{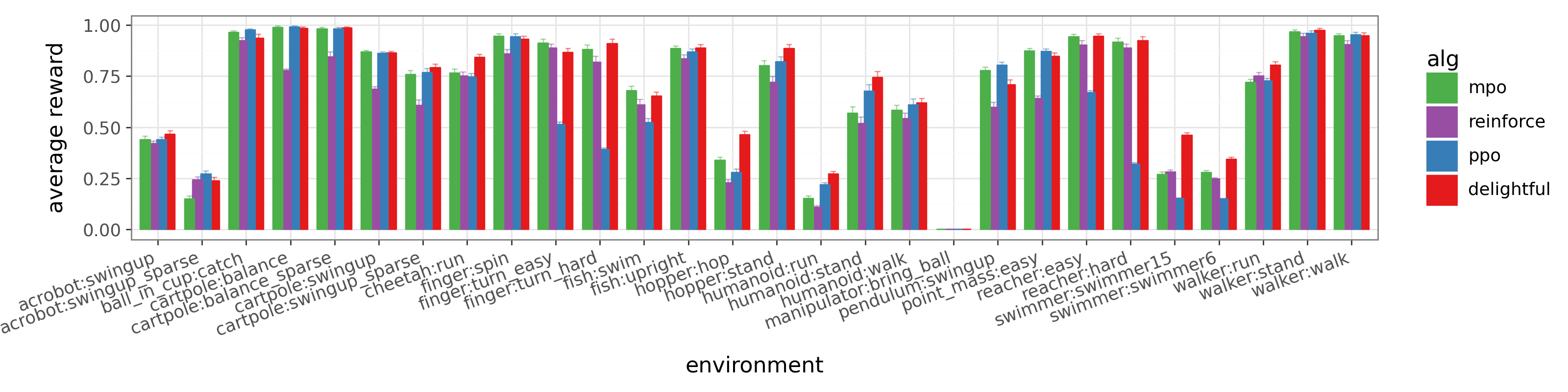}
\vspace{-5mm}
\caption{Average reward across 28 Control Suite environments.
DG (red) is never the worst method.}
\vspace{-1mm}
\label{fig:reward_complete}
\end{figure}

\noindent
\vspace{-2mm}
\begin{minipage}[t]{0.48\columnwidth}
  \centering
  \vspace{0pt}
  \includegraphics[width=\linewidth]{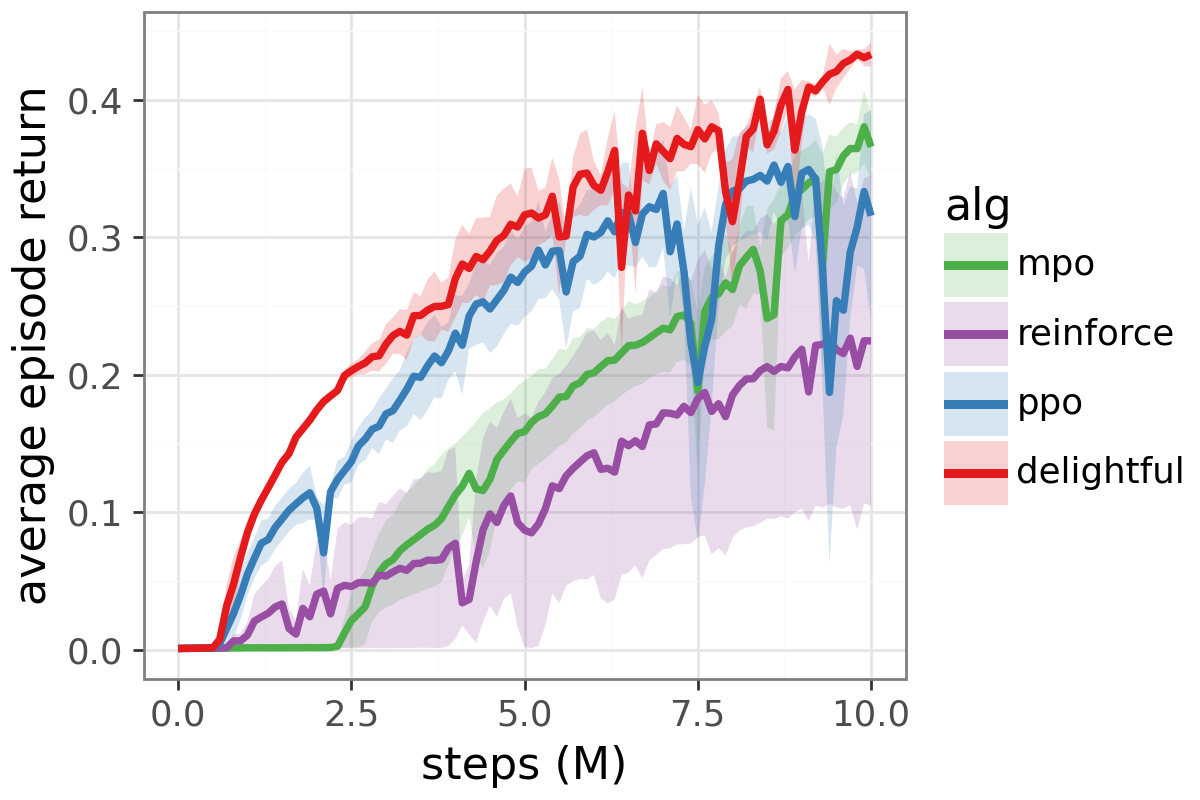}
  \vspace{-4mm}
  \captionof{figure}{Learning curve on \texttt{humanoid:run}.}
  \label{fig:reward_humanoid_run}
\end{minipage}
\hfill
\begin{minipage}[t]{0.48\columnwidth}
  \centering
  \vspace{0pt}
  \includegraphics[width=\linewidth]{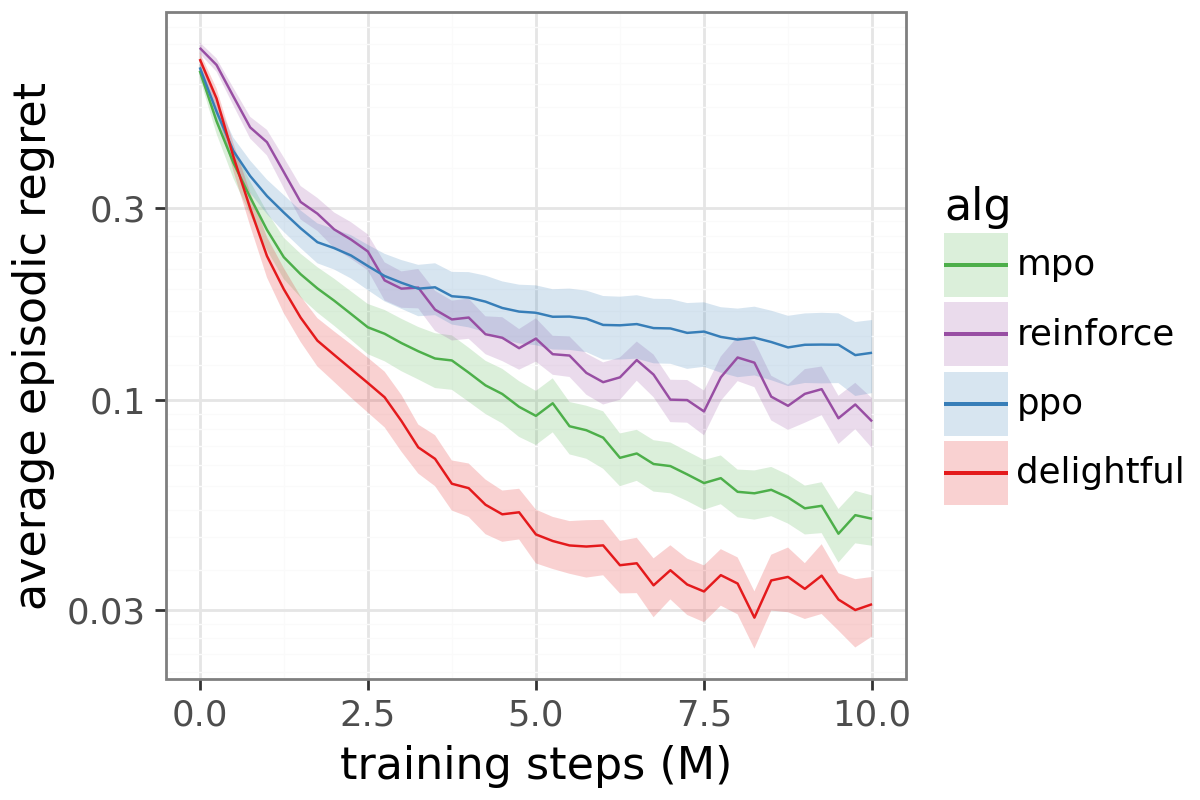}
  \vspace{-4mm}
  \captionof{figure}{Aggregate regret across 28 tasks.}
  \label{fig:regret_ave_control}
\end{minipage}

\section{Related Work}
\label{sec:related}

Many policy-gradient methods can be viewed through the effective score weight $\omega_t$ multiplying $\nabla_\theta \log \pi_\theta$.
This lens makes clear which methods distinguish rare from common actions and which do not (Figure~\ref{fig:compare_gradients}).
For standard PG, $\omega_t = U_t$ is linear in advantage and blind to action probability.
For DG, $\omega_t = \sigma(U_t \ell_t / \eta)\, U_t$ depends on both advantage and surprisal, allowing rare successes and rare failures to be treated asymmetrically.

\textbf{Trust-region methods.}
The natural policy gradient~\citep{kakade2001natural} and its trust-region successors TRPO, PPO~\citep{schulman2015trust,schulman2017proximal}, and GRPO~\citep{shao2024deepseekmath} constrain or precondition the update, keeping $\omega_t$ in a narrow band around $U_t$.
This limits large updates but also attenuates rare high-advantage events.
DG instead modulates $\omega_t$ using only the current policy and requires no importance ratios.

\textbf{Variance reduction.}
Leave-one-out baselines~\citep{kool2019buy}, generalized advantage estimation~\citep{schulman2016high}, and other control-variate methods reduce gradient variance without changing the expected direction.
DG's single-context mechanism (Prop.~\ref{prop:variance}) achieves a related effect through gating rather than control variates, but its cross-context rebalancing (Prop.~\ref{prop:direction}) changes the expected direction itself---an effect that persists even with perfect variance reduction.

\textbf{Advantage-weighted methods.}
AWR~\citep{peng2019advantage} and PMPO~\citep{abdolmaleki2024preference} set $\omega_t = f(U_t)$ for some increasing function of advantage, but remain blind to surprisal.
Gradient budget therefore continues to flow toward predictable actions until their advantage is driven toward zero.
DG makes $\omega_t$ depend on surprisal directly, so rare and common actions are weighted differently even at the same advantage.

\textbf{Entropy and intrinsic motivation.}
Entropy regularization~\citep{haarnoja2018soft} and curiosity-based methods~\citep{pathak2017curiosity} change the learning signal indirectly by modifying the reward or objective.
DG instead leaves the reward unchanged and modifies the policy-gradient coefficient itself, reallocating gradient budget toward surprising task-relevant outcomes.
Its distinct behavior relative to entropy bonuses is confirmed empirically in the MNIST experiments (Section~\ref{sec:mnist}).

\begin{figure}[ht]
\centering
\vspace{-2mm}
\includegraphics[width=\columnwidth]{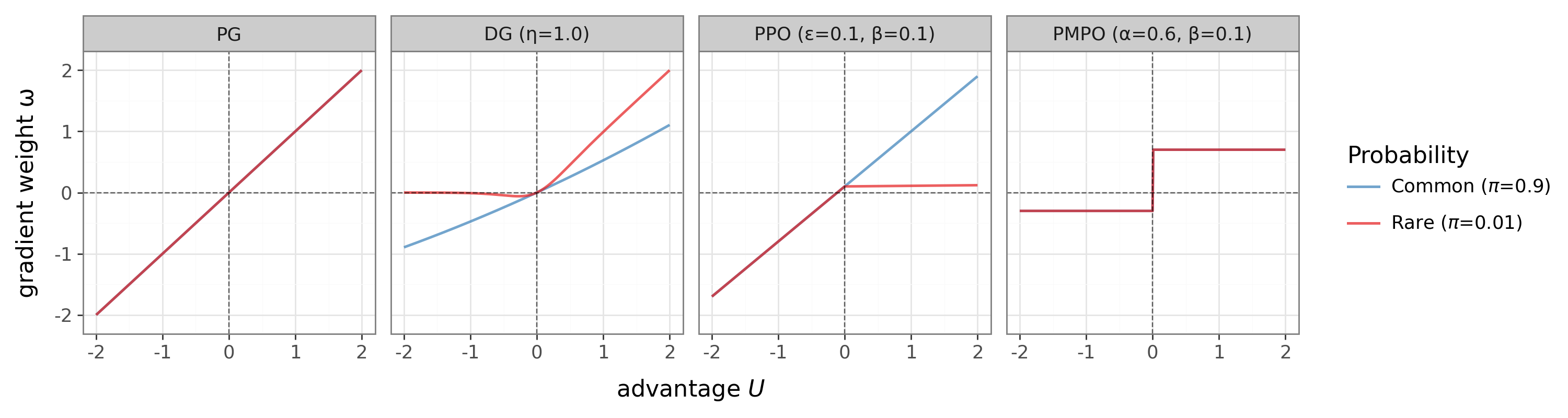}
\vspace{-6mm}
\caption{Score weight $\omega$ vs.\ advantage $U$ for common ($\pi{=}0.9$) and rare ($\pi{=}0.01$) actions.
DG treats rare successes and rare failures asymmetrically; PG and PMPO are blind to action probability, while PPO clips both tails through importance-ratio constraints.}
\label{fig:compare_gradients}
\end{figure}
\vspace{-2mm}

\section{Conclusion}
\label{sec:conclusion}
\vspace{-1mm}

This paper identifies a basic mismatch in standard policy-gradient learning.
In the K-armed bandits we analyze, rare failures inject disproportionate noise within a context, and the expected gradient over-allocates budget to contexts the policy already handles well.
MNIST, transformer, and control experiments suggest both effects persist beyond the tabular setting.
DG corrects both effects by gating each score term with delight, the product of advantage and surprisal, amplifying rare successes and suppressing rare failures.
The first effect, variance reduction, vanishes with infinite samples; the second, a beneficial bias toward the cross-entropy oracle, does not.
We characterize both analytically and confirm them empirically across bandits, MNIST, transformer sequence modeling, and continuous control.

More broadly, these results suggest that delight is a useful primitive for decision learning under evaluative feedback.
DG does more than reduce variance: it changes how finite update budget is allocated across samples and contexts.
From this perspective, the method reveals that standard policy gradients use a suboptimal weighting rule for learning from evaluative feedback.
Formal convergence guarantees remain open, as does the question of how far this mechanism transfers to sparse-reward settings, offline RL, and large-scale transformer training and RLHF.

\section*{Acknowledgements}
\vspace{-1mm}

We thank Ben Van Roy, Satinder Singh, Tor Lattimore, and Jincheng Mei for detailed reviews and feedback on earlier drafts.
John Aslanides, Yotam Doron, Georg Ostrovski, Hubert Soyer, and Blanca Huergo contributed to the codebase and helped shape the experimental infrastructure.
We are grateful to Raia Hadsell, Zoubin Ghahramani, Demis Hassabis, and Satinder Singh for fostering the research environment that made this work possible.
Finally, we thank Dan Zigmond and Barry Kerzin for inspiration from Buddhist philosophy and feedback on the wider research into delightful learning.

\bibliography{references}

@inproceedings{abdolmaleki2018maximum,
  author = {Abdolmaleki, Abbas and Springenberg, Jost Tobias and Tassa, Yuval and Munos, Remi and Heess, Nicolas and Riedmiller, Martin},
  title = {Maximum a Posteriori Policy Optimisation},
  booktitle = {International Conference on Learning Representations},
  year = {2018},
}

@article{abdolmaleki2024preference,
  author = {Abdolmaleki, Abbas and Piot, Bilal and Shahriari, Bobak and Springenberg, Jost Tobias and Hertweck, Tim and Joshi, Rishabh and Oh, Junhyuk and Bloesch, Michael and Lampe, Thomas and Heess, Nicolas and others},
  title = {Preference optimization as probabilistic inference},
  journal = {arXiv e-prints},
  pages = {arXiv--2410},
  year = {2024},
}

@inproceedings{haarnoja2018soft,
  author = {Haarnoja, Tuomas and Zhou, Aurick and Abbeel, Pieter and Levine, Sergey},
  title = {Soft Actor-Critic: Off-Policy Maximum Entropy Deep Reinforcement Learning with a Stochastic Actor},
  booktitle = {International Conference on Machine Learning},
  pages = {1861--1870},
  year = {2018},
}

@article{kelly1956new,
  author = {Kelly, John L.},
  title = {A New Interpretation of Information Rate},
  journal = {Bell System Technical Journal},
  volume = {35},
  number = {4},
  pages = {917--926},
  year = {1956},
}

@inproceedings{kakade2001natural,
  author = {Kakade, Sham M},
  title = {A Natural Policy Gradient},
  booktitle = {Advances in Neural Information Processing Systems 14},
  year = {2001},
}

@inproceedings{kool2019buy,
  author = {Kool, Wouter and van Hoof, Herke and Welling, Max},
  title = {Buy 4 {REINFORCE} samples, get a baseline for free!},
  booktitle = {Deep Reinforcement Learning Meets Structured Prediction, ICLR Workshop},
  year = {2019},
}

@article{kingma2014adam,
  author = {Kingma, Diederik P. and Ba, Jimmy},
  title = {Adam: A Method for Stochastic Optimization},
  booktitle = {Proc. of ICLR},
  year = {2015},
}

@inproceedings{munos2016safe,
  author = {Munos, R{\'e}mi and Stepleton, Tom and Harutyunyan, Anna and Bellemare, Marc},
  title = {Safe and efficient off-policy reinforcement learning},
  booktitle = {Advances in Neural Information Processing Systems 29},
  pages = {1046--1054},
  year = {2016},
}

@inproceedings{pathak2017curiosity,
  author = {Pathak, Deepak and Agrawal, Pulkit and Efros, Alexei A and Darrell, Trevor},
  title = {Curiosity-driven exploration by self-supervised prediction},
  booktitle = {International Conference on Machine Learning},
  pages = {2778--2787},
  year = {2017},
}

@article{peng2019advantage,
  author = {Peng, Xue Bin and Kumar, Aviral and Zhang, Grace and Levine, Sergey},
  title = {Advantage-weighted regression: Simple and scalable off-policy reinforcement learning},
  journal = {arXiv preprint arXiv:1910.00177},
  year = {2019},
}

@inproceedings{schulman2015trust,
  author = {Schulman, John and Levine, Sergey and Abbeel, Pieter and Jordan, Michael and Moritz, Philipp},
  title = {Trust Region Policy Optimization},
  booktitle = {Proc. of ICML},
  year = {2015},
}

@article{schulman2017proximal,
  author = {Schulman, John and Wolski, Filip and Dhariwal, Prafulla and Radford, Alec and Klimov, Oleg},
  title = {Proximal Policy Optimization Algorithms},
  journal = {arXiv preprint arXiv:1707.06347},
  year = {2017},
}

@article{schulman2016high,
  author = {Schulman, John and Moritz, Philipp and Levine, Sergey and Jordan, Michael and Abbeel, Pieter},
  title = {High-Dimensional Continuous Control Using Generalized Advantage Estimation},
  journal = {Proc. of ICLR},
  year = {2016},
}

@article{shao2024deepseekmath,
  author = {Shao, Zhihong and Wang, Peiyi and Zhu, Qihao and Xu, Runxin and Song, Junxiao and Zhang, Mingchuan and Li, Y.K. and Wu, Y. and Guo, Daya},
  title = {{DeepSeekMath}: Pushing the limits of mathematical reasoning in open language models},
  journal = {arXiv preprint arXiv:2402.03300},
  year = {2024},
}

@article{ouyang2022training,
  author = {Ouyang, Long and Wu, Jeffrey and Jiang, Xu and Almeida, Diogo and Wainwright, Carroll and Mishkin, Pamela and Zhang, Chong and Agarwal, Sandhini and Slama, Katarina and Ray, Alex and others},
  title = {Training language models to follow instructions with human feedback},
  journal = {Advances in Neural Information Processing Systems},
  volume = {35},
  pages = {27730--27744},
  year = {2022},
}

@article{silver2017mastering,
  author = {Silver, David and Schrittwieser, Julian and Simonyan, Karen and Antonoglou, Ioannis and Huang, Aja and Guez, Arthur and Hubert, Thomas and Baker, Lucas and Lai, Matthew and Bolton, Adrian and others},
  title = {Mastering the game of {Go} without human knowledge},
  journal = {Nature},
  volume = {550},
  number = {7676},
  pages = {354--359},
  year = {2017},
}

@article{vinyals2019grandmaster,
  author = {Vinyals, Oriol and Babuschkin, Igor and Czarnecki, Wojciech M and Mathieu, Micha{\"e}l and Dudzik, Andrew and Chung, Junyoung and Choi, David H and Powell, Richard and Ewalds, Timo and Georgiev, Petko and others},
  title = {Grandmaster level in {StarCraft II} using multi-agent reinforcement learning},
  journal = {Nature},
  volume = {575},
  number = {7782},
  pages = {350--354},
  year = {2019},
}

@inproceedings{sutton1999policy,
  author = {Sutton, Richard S and McAllester, David A and Singh, Satinder P and Mansour, Yishay},
  title = {Policy Gradient Methods for Reinforcement Learning with Function Approximation.},
  booktitle = {Proc. of NIPS},
  year = {1999},
}

@article{tassa2018deepmind,
  author = {Tassa, Yuval and Doron, Yotam and Muldal, Alistair and Erez, Tom and Li, Yazhe and Casas, Diego de Las and Budden, David and Abdolmaleki, Abbas and Merel, Josh and Lefrancq, Andrew and others},
  title = {DeepMind Control Suite},
  journal = {arXiv preprint arXiv:1801.00690},
  year = {2018},
}

@article{williams1992simple,
  author = {Williams, Ronald J},
  title = {Simple statistical gradient-following algorithms for connectionist reinforcement learning},
  journal = {Machine learning},
  publisher = {Springer},
  volume = {8},
  number = {3},
  pages = {229--256},
  year = {1992},
}

@inproceedings{you2019large,
  author = {You, Yang and Li, Jing and Reddi, Sashank and Hseu, Jonathan and Kumar, Sanjiv and Bhojanapalli, Srinadh and Song, Xiaodan and Demmel, James and Keutzer, Kurt and Hsieh, Cho-Jui},
  title = {Large Batch Optimization for Deep Learning: Training {BERT} in 76 Minutes},
  booktitle = {International Conference on Learning Representations},
  year = {2020},
}
\bibliographystyle{rlj}

\beginSupplementaryMaterials

\appendix

\section{The Delightful Gate: Derivation and Properties}
\label{app:derivations}

We derive the sigmoid gate and establish basic properties.

\paragraph{Entropy-regularized gate selection.}
Treat each sampled score term $g_t$ as a candidate update that is applied with weight $w \in [0,1]$.
For fixed delight $\chi$, we choose $w$ to maximize a linear reward plus an entropy cost:
\begin{equation}
\max_{w \in [0,1]} \; \chi\,w + \eta\,H(w),
\qquad
H(w) = -w\log w - (1{-}w)\log(1{-}w).
\end{equation}
Differentiating and setting to zero:
\begin{equation}
\frac{\partial}{\partial w}\bigl[\chi w + \eta H(w)\bigr] 
= \chi - \eta \log\!\left(\frac{w}{1-w}\right) = 0
\quad\Longrightarrow\quad
w^* = \sigma\!\left(\frac{\chi}{\eta}\right).
\end{equation}
Since $H$ is strictly concave, this is the unique global maximizer.

\paragraph{Softplus potential.}
Substituting $w^*$ into the objective yields the optimal value:
\begin{equation}
\Psi_\eta(\chi) = \max_{w \in [0,1]} \bigl\{\chi w + \eta H(w)\bigr\} = \eta \log\!\left(1 + e^{\chi/\eta}\right) = \eta\,\mathrm{softplus}(\chi/\eta).
\end{equation}
This softplus potential provides a local variational interpretation of the DG gate.
It saturates for $\chi \ll 0$ (suppressing updates from disfavored actions) and grows linearly for $\chi \gg 0$ (preserving updates from surprising successes).

\paragraph{Temperature interpolation.}
The temperature $\eta$ controls gate sharpness and interpolates between two regimes.
In the \emph{hedonic} limit ($\eta \to \infty$), $w^* \to 1/2$ for bounded $\chi$, recovering standard PG up to a constant: the gate ignores surprisal and responds only to advantage.
In the \emph{enlightened} limit ($\eta \to 0$), $w^* \to \mathbb{I}\{\chi > 0\}$, a hard gate that passes all positive-delight samples equally and fully suppresses negative ones.
We use $\eta = 1$ throughout, which provides meaningful asymmetry without hard thresholding.

\paragraph{Stationarity at optimal policies.}
At an optimal policy $\pi^*$, on-policy advantages vanish on the support.
Since the DG update is $w \cdot U \cdot \nabla_\theta \log \pi_\theta$ and $U = 0$ on supported actions, the update vanishes.
Optimal policies are therefore stationary points of DG for all $\eta > 0$.

\section{MNIST Experimental Details}
\label{app:mnist}

We provide supplementary details and robustness analyses for the MNIST contextual bandit experiments in Section~\ref{sec:mnist}.

\paragraph{Setup.}
We parameterize the policy $\pi_\theta$ as a two-layer MLP with ReLU activations and hidden width 100.
Training proceeds in online batches: each step, the agent collects $B=100$ trajectories, computes gradients, and updates parameters with Adam (learning rate $10^{-3}$).
We train for $K=10{,}000$ episodes unless otherwise noted.

\subsection{Generalization}
\label{app:mnist_generalization}

A natural concern is that prioritizing high-surprisal events might cause overfitting.
Figure~\ref{fig:mnist_validation} shows validation error tracks training error closely for both PG and DG, and the optimal temperature $\eta \approx 1$ is consistent across splits.
DG's gains persist on held-out data.

\begin{figure}[ht]
\centering
\begin{subfigure}[b]{0.48\columnwidth}
    \includegraphics[width=\linewidth]{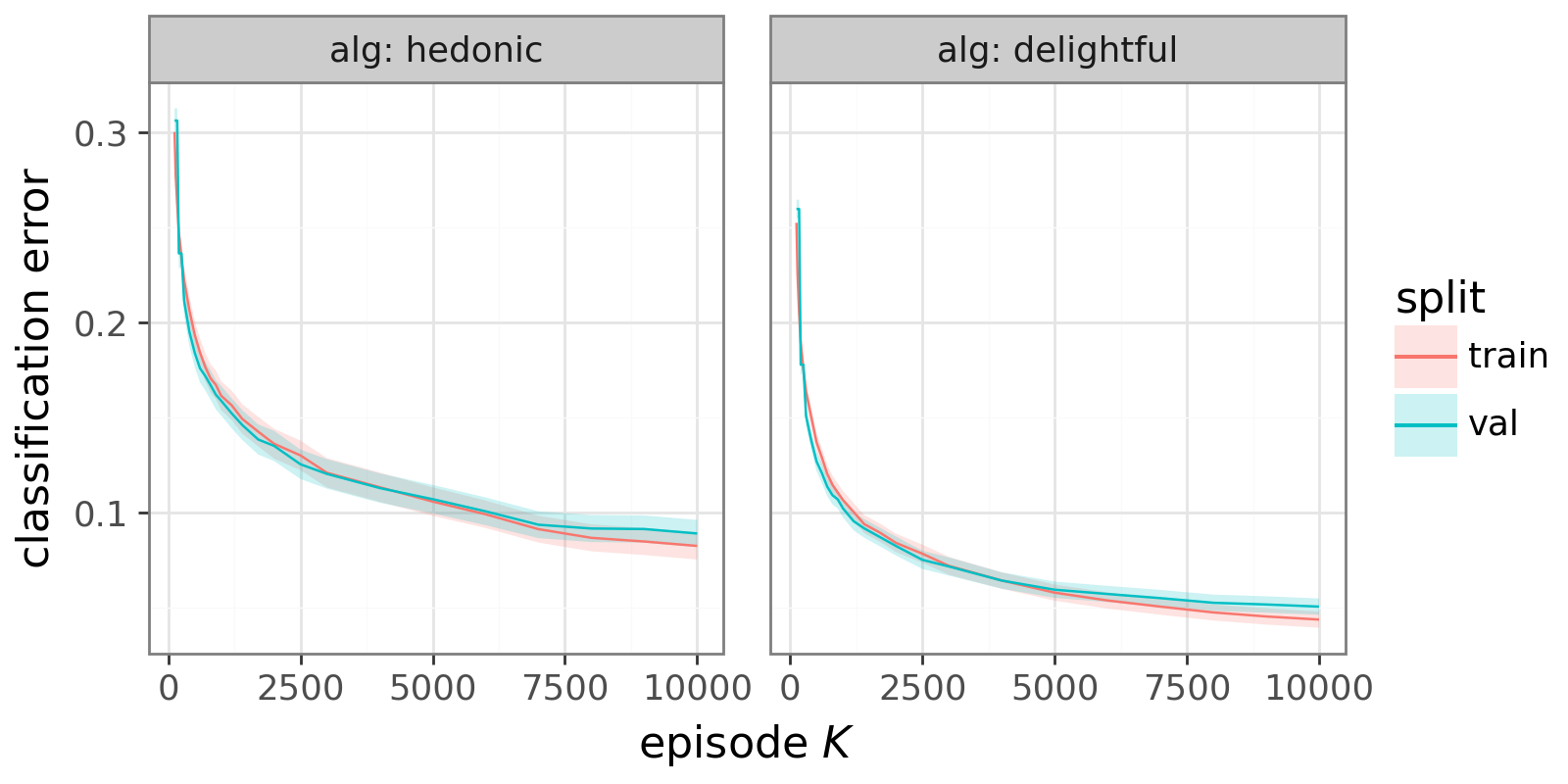}
    \caption{Learning curves (train vs.\ val).}
    \label{fig:mnist_curve_val}
\end{subfigure}
\hfill
\begin{subfigure}[b]{0.48\columnwidth}
    \includegraphics[width=\linewidth]{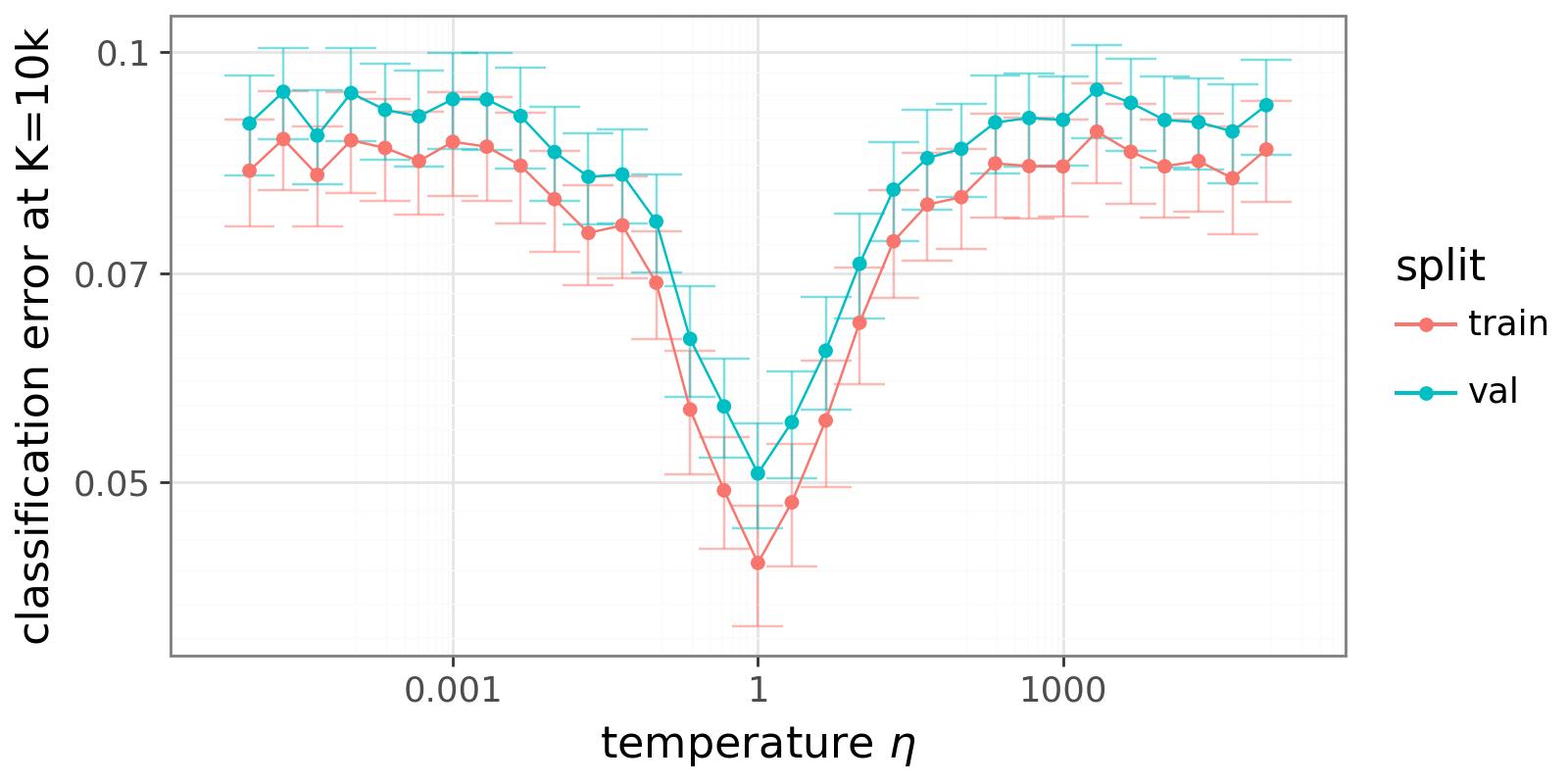}
    \caption{Temperature sweep (train vs.\ val).}
    \label{fig:mnist_eta_val}
\end{subfigure}
\caption{Generalization on MNIST.
\textbf{(a)}~Validation error tracks training error; DG achieves lower error on both.
\textbf{(b)}~The optimal temperature $\eta \approx 1$ is consistent across train and validation.}
\label{fig:mnist_validation}
\end{figure}

\subsection{Comparison with Entropy Regularization}
\label{app:mnist_entropy}

A natural alternative is additive entropy regularization, augmenting the objective with $\alpha H(\pi_\theta)$.
Figure~\ref{fig:mnist_entropy_app} sweeps $\alpha$ for PG and $\eta$ for DG.
Entropy regularization is sensitive to its coefficient: small $\alpha$ provides negligible benefit, while large $\alpha$ collapses performance by forcing excessive stochasticity.
DG achieves a lower error floor than any entropy-regularized baseline and maintains its advantage across a wide range of $\eta$.

\begin{figure}[ht]
\centering
\includegraphics[width=0.6\columnwidth]{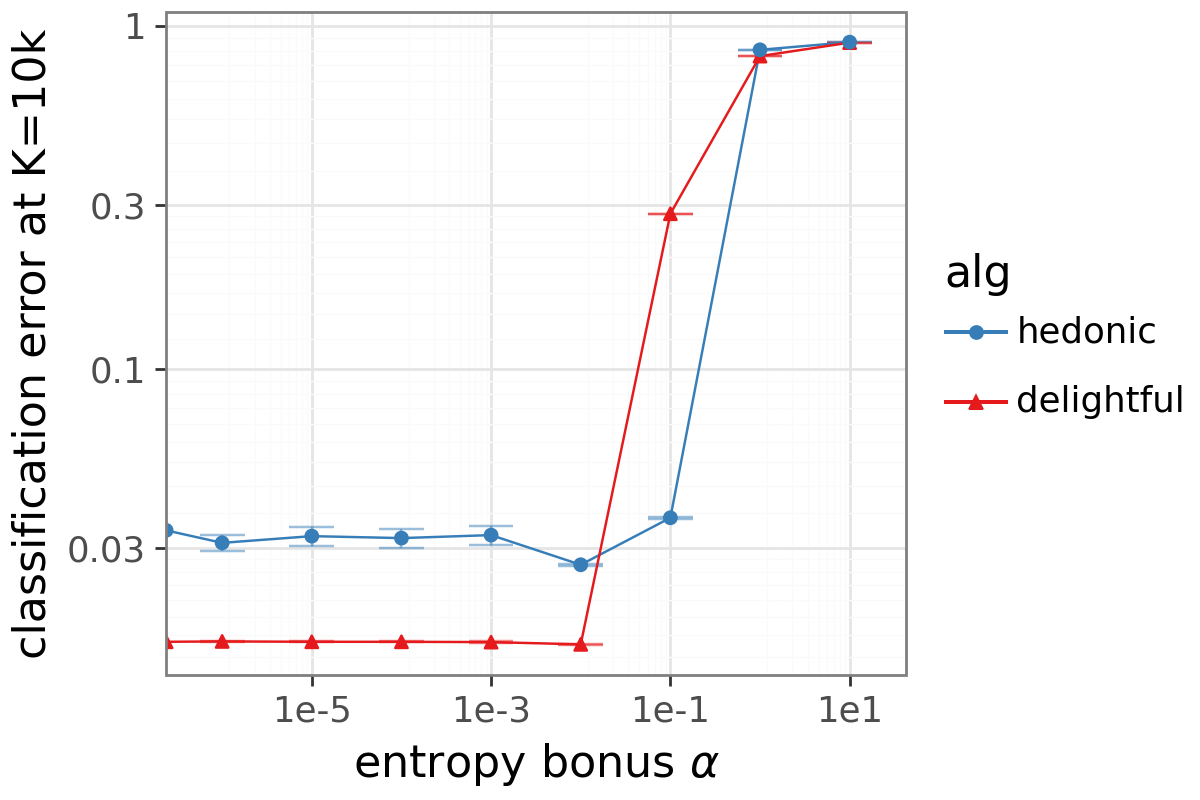}
\caption{Entropy regularization vs.\ DG temperature.
Large entropy coefficients degrade accuracy; DG is robust across $\eta$ and consistently achieves lower error.}
\label{fig:mnist_entropy_app}
\end{figure}

\subsection{Robustness to Learning Rate}
\label{app:mnist_lr}

Figure~\ref{fig:mnist_lr} sweeps learning rate $\alpha \in [10^{-5}, 10^{-2}]$.
DG consistently outperforms PG across the effective range, with a wider basin of low error around $\alpha \in [10^{-4}, 10^{-3}]$.

\begin{figure}[ht]
\centering
\includegraphics[width=\columnwidth]{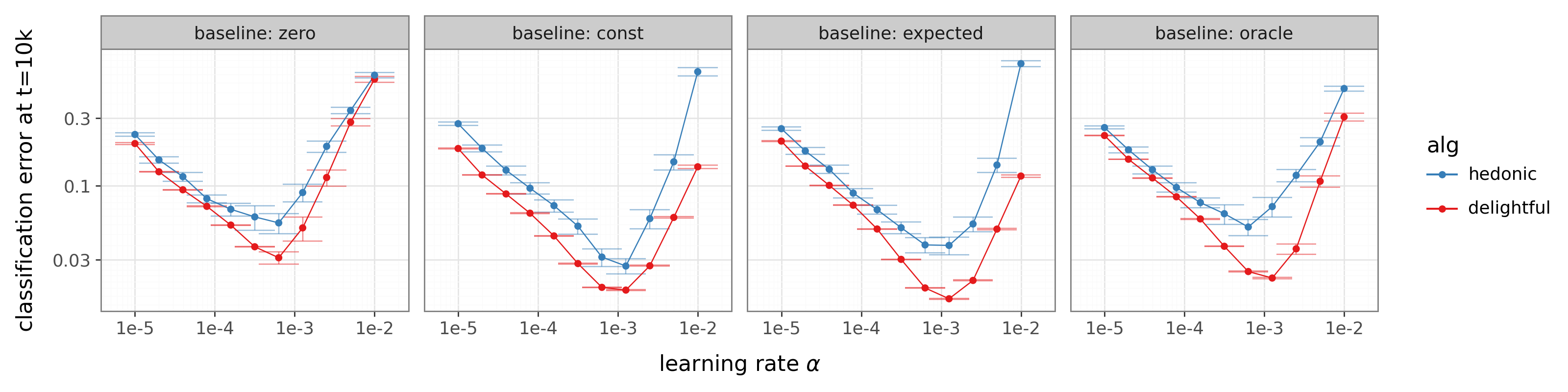}
\caption{Learning rate sensitivity.
DG (red) outperforms PG (blue) across all baselines and learning rates.}
\label{fig:mnist_lr}
\end{figure}

\subsection{Robustness to Batch Size}
\label{app:mnist_batch}

Figure~\ref{fig:mnist_batch} sweeps batch size $B \in [1, 1000]$.
Larger batches reduce gradient variance and improve both methods, but DG maintains a consistent advantage.
This confirms that DG's benefit is not merely variance reduction---it provides a distinct signal that complements larger batches.

\begin{figure}[ht]
\centering
\includegraphics[width=\columnwidth]{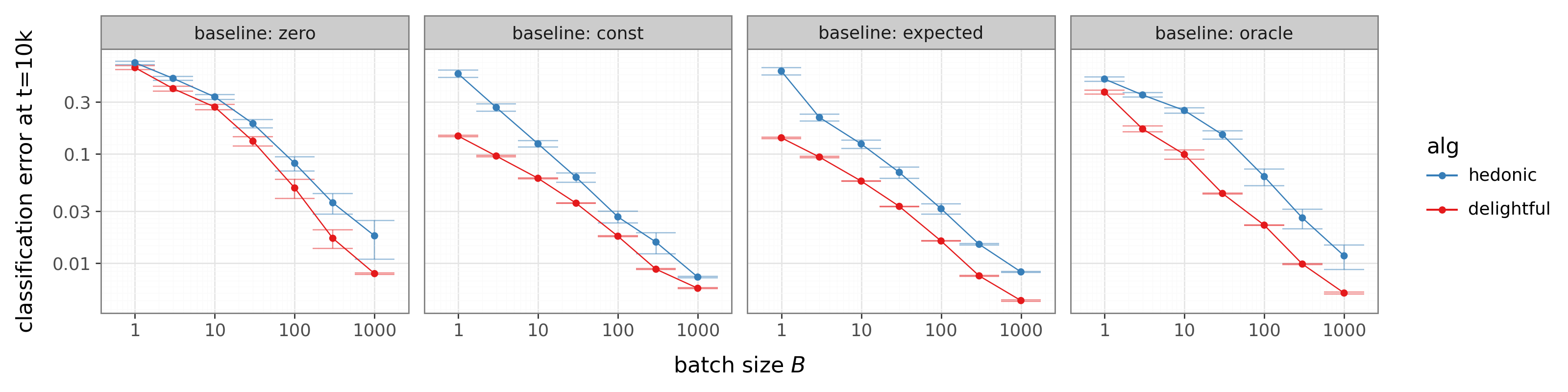}
\caption{Batch size sensitivity.
DG maintains its advantage across all batch sizes.}
\label{fig:mnist_batch}
\end{figure}

\subsection{Robustness to Network Width}
\label{app:mnist_width}

Figure~\ref{fig:mnist_width} varies hidden width $W \in [1, 2048]$.
Both methods exhibit a sweet spot in capacity, but DG achieves lower error at every width.

\begin{figure}[ht]
\centering
\includegraphics[width=\columnwidth]{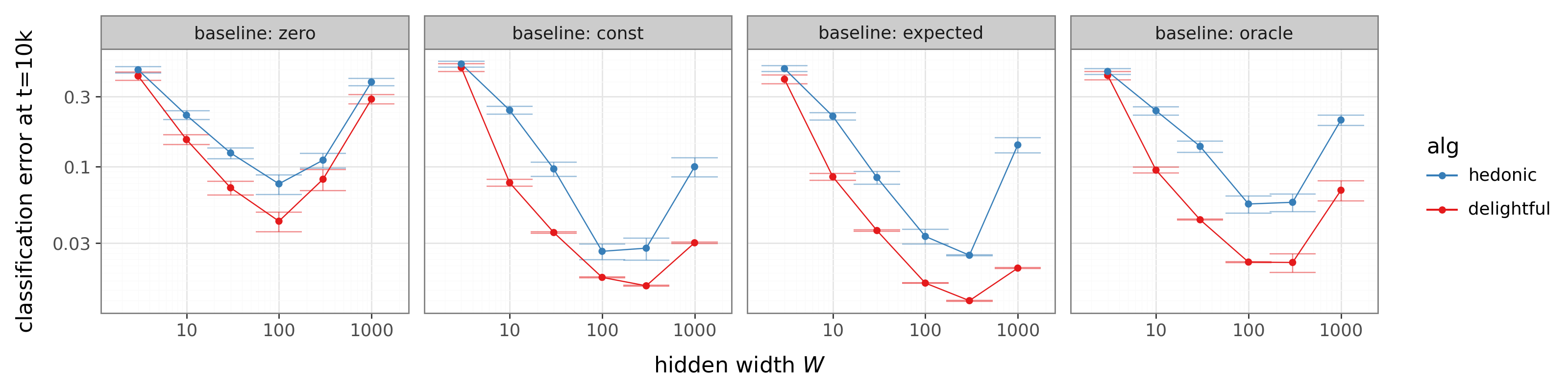}
\caption{Network width sensitivity.
DG outperforms PG across all capacity levels.}
\label{fig:mnist_width}
\end{figure}

\subsection{Alternative Definitions of Delight}
\label{app:mnist_structure}

We investigate whether alternative functional forms could improve on $\chi = U \cdot \ell$.

\paragraph{Additive mixtures.}
Inspired by UCB, we consider $\chi^{\text{UCB}}_\alpha = (1-\alpha) U + \alpha \ell$, interpolating between pure advantage ($\alpha=0$) and pure surprisal ($\alpha=1$).
Figure~\ref{fig:mnist_structure_ablation}(a) shows additive mixtures consistently underperform the multiplicative form.

\paragraph{Surprisal exponent.}
We generalize to $\chi_\beta = U \cdot \ell^\beta$.
Figure~\ref{fig:mnist_structure_ablation}(b) shows the optimum is exactly $\beta=1$; deviations in either direction degrade performance.
Within this family of alternatives, the simple product $\chi = U \cdot \ell$ performs best.

\begin{figure}[ht]
\centering
\begin{subfigure}[b]{0.48\columnwidth}
    \includegraphics[width=\linewidth]{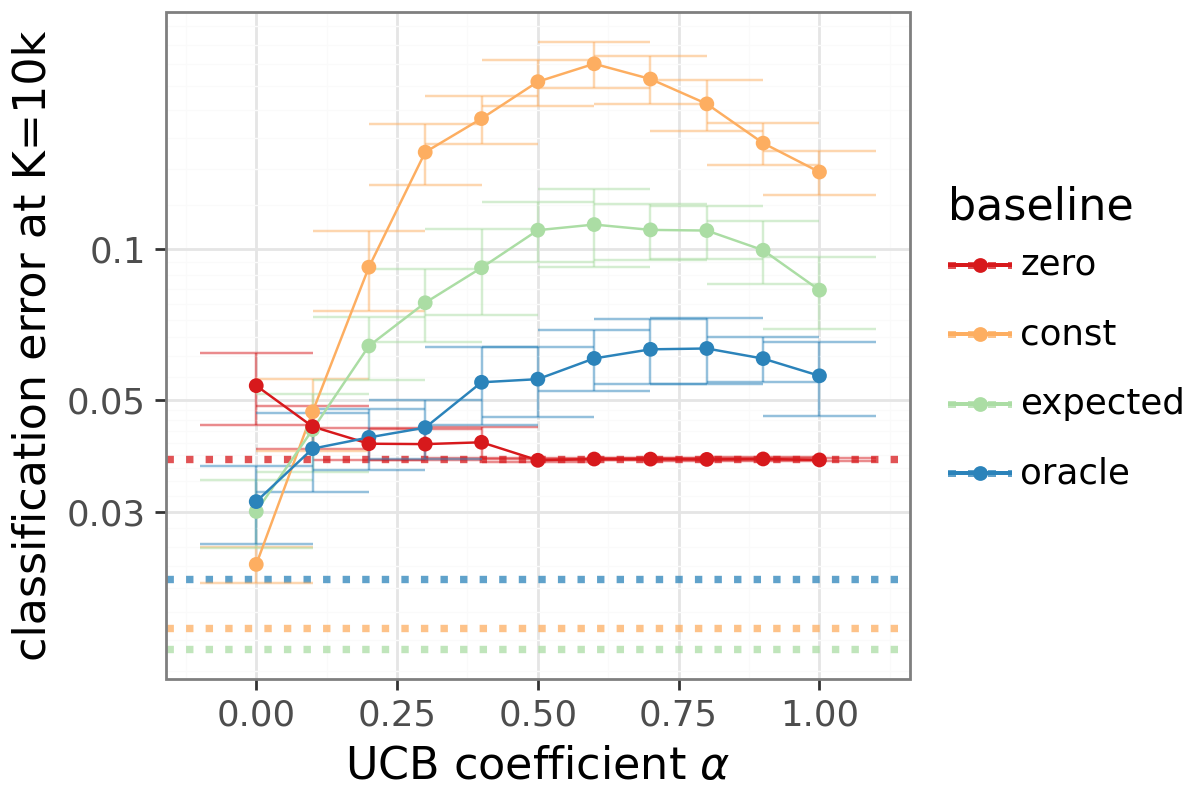}
    \caption{Additive mixtures.}
    \label{fig:mnist_ucb}
\end{subfigure}
\hfill
\begin{subfigure}[b]{0.48\columnwidth}
    \includegraphics[width=\linewidth]{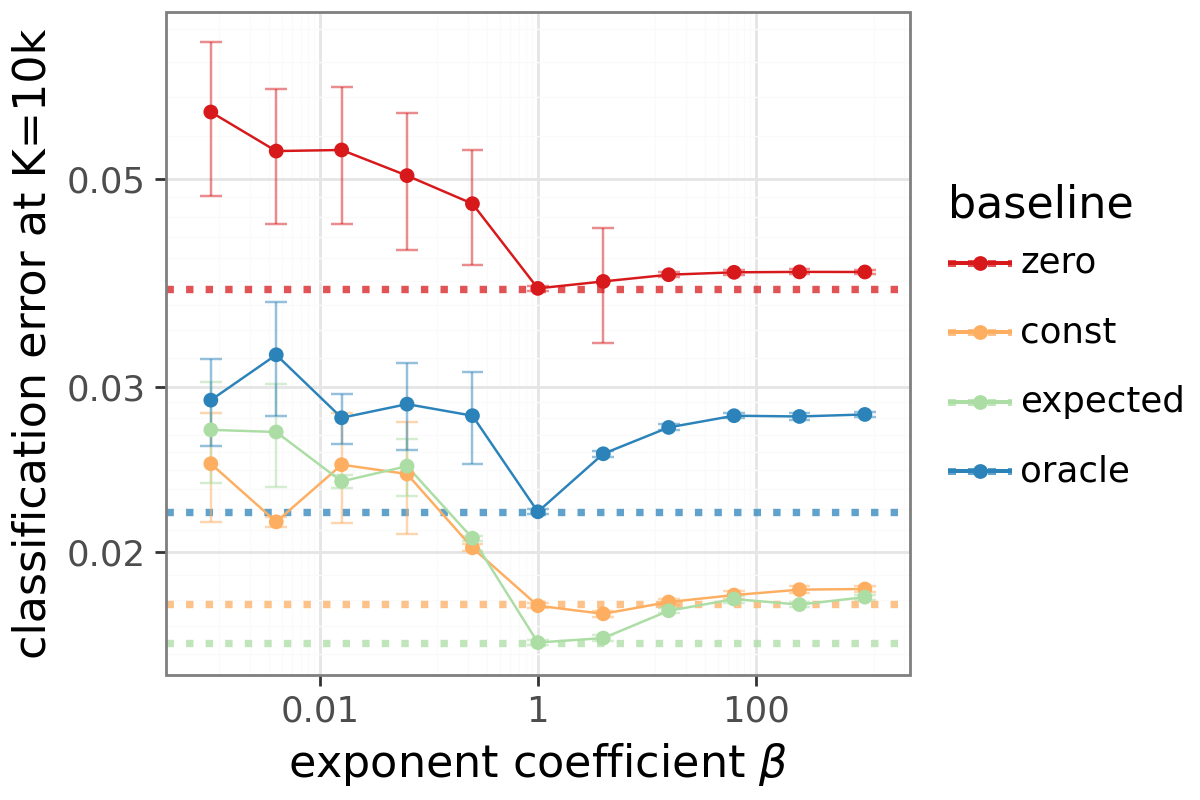}
    \caption{Surprisal exponent $\beta$.}
    \label{fig:mnist_exponent}
\end{subfigure}
\caption{Alternative definitions of delight.
\textbf{(a)}~Additive mixtures underperform the multiplicative form (dotted).
\textbf{(b)}~The simple product ($\beta=1$) is optimal.}
\label{fig:mnist_structure_ablation}
\end{figure}

\section{Proofs for Tabular Analysis}
\label{app:bandit_proofs}

This section collects the proofs for Section~\ref{sec:bandit}.
We first justify the normalized-step model, then prove Propositions~\ref{prop:variance} and~\ref{prop:direction} and their extensions.

\subsection{Cosine Controls Progress}
\label{app:cosine_progress}

The tabular analysis assumes normalized steps $z \leftarrow z + \alpha g/\|g\|$.
The following proposition shows that under this update rule, expected improvement scales linearly with $\cos(g, \nabla J)$, so a higher-cosine gradient estimator translates directly into faster learning.

\begin{proposition}[Cosine controls progress]
\label{prop:cosine_progress}
Let $J$ have $L$-Lipschitz gradient.
For the normalized update $z^+ = z + \alpha g / \|g\|$,
\[
\mathbb{E}[J(z^+) - J(z) \mid z]
\ge \alpha \|\nabla J(z)\| \cdot \mathbb{E}[\cos(g, \nabla J(z)) \mid z] - \tfrac{L}{2}\alpha^2.
\]
\end{proposition}

\begin{proof}
By smoothness: $J(z^+) \ge J(z) + \langle \nabla J, \alpha g/\|g\| \rangle - \frac{L}{2}\alpha^2$.
Taking expectations and using $\langle \nabla J, g/\|g\| \rangle = \|\nabla J\| \cos(g, \nabla J)$.
\end{proof}

\subsection{Proof of Proposition~\ref{prop:variance} (Variance Reduction)}
\label{app:prop_a}

We use the symmetric bandit setup from Section~\ref{subsec:variance}: $K$ actions, correct action $y^*$, $\pi(y^*) = 1 - \varepsilon$, $\pi(a) = q := \varepsilon/(K-1)$ for $a \neq y^*$, baseline $b$, and gate values $w_+$ (correct) and $w_-$ (incorrect).
The proof proceeds in three parts: (i)~a symmetry lemma shows that all expected gradients are collinear, (ii)~the perpendicular variance factors exactly, and (iii)~a Taylor expansion converts the variance ratio into a cosine-gap ratio.

\paragraph{Part (i): Direction preservation.}

Let $u := \phi(y^*) = e_{y^*} - \pi$.
We show $\sum_{a \neq y^*} \phi(a)$ is proportional to $u$, which forces the expected DG gradient to be collinear with $\nabla J = (1 - \varepsilon)\,u$.

\begin{lemma}[Symmetry identity]
\label{lem:symmetry}
$\displaystyle \sum_{a \neq y^*} \phi(a) = -\frac{(K-1)(1-\varepsilon)}{\varepsilon}\,\phi(y^*).$
\end{lemma}

\begin{proof}
Since $\sum_{a=1}^K \phi(a) = \sum_a (e_a - \pi) = \mathbf{1} - K\pi$, we have $\sum_{a \neq y^*} \phi(a) = \mathbf{1} - K\pi - \phi(y^*)$.
We verify this is proportional to $\phi(y^*)$ component by component.
In component $y^*$: $1 - K(1-\varepsilon) - \varepsilon = -(K-1)(1-\varepsilon)$.
In component $a \neq y^*$: $1 - Kq - (-q) = 1 - (K-1)q = 1 - \varepsilon$.
Since $\phi(y^*)$ has component $\varepsilon$ in position $y^*$ and $-q = -\varepsilon/(K-1)$ in other positions, the ratio is $-(K-1)(1-\varepsilon)/\varepsilon$ in both components.
\end{proof}

Now compute the expected DG gradient:
\begin{align*}
\mathbb{E}[g_{\mathrm{DG}}]
&= (1-\varepsilon)\,w_+\,(1-b)\,\phi(y^*) + \sum_{a \neq y^*} q\,w_-\,(-b)\,\phi(a) \\
&= (1-\varepsilon)\,w_+\,(1-b)\,\phi(y^*) + w_-\,(-b)\,q \cdot \Big(-\frac{(K-1)(1-\varepsilon)}{\varepsilon}\Big)\phi(y^*) \\
&= (1-\varepsilon)\Big[(1-b)\,w_+ + b\,w_-\Big]\,\phi(y^*) \\
&= s \cdot \nabla J,
\end{align*}
where $s = (1-b)\,w_+ + b\,w_-$ and we used $q(K-1) = \varepsilon$.
Since $w_+, w_- > 0$ and $b \in (0,1)$, we have $s > 0$, confirming the direction is preserved.

\paragraph{Part (ii): Perpendicular variance is exactly $w_-^2 \cdot \mathrm{Var}_\perp(g_{\mathrm{PG}})$.}

Since $\nabla J \propto \phi(y^*)$, the projection $\Pi_\perp \phi(y^*) = 0$.
Therefore the correct-action sample contributes zero perpendicular energy under both PG and DG.
For incorrect actions, $g_{\mathrm{DG}}(a) = w_- \cdot g_{\mathrm{PG}}(a)$ pointwise, so $\Pi_\perp(g_{\mathrm{DG}}(a)) = w_- \cdot \Pi_\perp(g_{\mathrm{PG}}(a))$.

Note that $\mathbb{E}[\Pi_\perp(g)] = \Pi_\perp(\mathbb{E}[g]) = 0$ for both PG and DG (from Part~(i), both means are parallel to $\nabla J$), so $\mathrm{Var}_\perp$ equals the second moment.
The perpendicular variance is:
\begin{align*}
\mathrm{Var}_\perp(g_{\mathrm{DG}})
&= \mathbb{E}\,\|\Pi_\perp(g_{\mathrm{DG}})\|^2
= \sum_{a \neq y^*} \pi(a)\,w_-^2\,b^2\,\|\Pi_\perp(\phi(a))\|^2 \\
&= w_-^2 \sum_{a \neq y^*} \pi(a)\,b^2\,\|\Pi_\perp(\phi(a))\|^2
= w_-^2 \cdot \mathrm{Var}_\perp(g_{\mathrm{PG}}).
\end{align*}

\paragraph{Part (iii): Alignment gap ratio.}

When $\bar{g}$ concentrates near $\mu := \mathbb{E}[\bar{g}]$, write $\bar{g} = \mu + \xi$ with $\xi$ small.
Since $\mu$ is parallel to $\nabla J$ (Part~(i)), $\cos(\bar{g}, \nabla J) = \cos(\mu + \xi, \mu)$.
Taylor-expanding to second order in $\xi$:
\[
1 - \cos(\mu + \xi, \mu) \approx \frac{\|\xi_\perp\|^2}{2\|\mu\|^2},
\]
where $\xi_\perp = \Pi_\perp(\xi)$ is the perpendicular component.
Taking expectations:
\[
1 - \mathbb{E}[\cos(\bar{g}, \nabla J)]
\approx \frac{\mathrm{Var}_\perp(\bar{g})}{2\,\|\mathbb{E}[\bar{g}]\|^2}
= \frac{\mathrm{Var}_\perp(g)}{2B\,\|\mathbb{E}[\bar{g}]\|^2}.
\]
For PG: $\|\mathbb{E}[\bar{g}_{\mathrm{PG}}]\| = \|\nabla J\|$.
For DG: $\|\mathbb{E}[\bar{g}_{\mathrm{DG}}]\| = s\,\|\nabla J\|$.
Using Part~(ii):
\[
\frac{1 - \mathbb{E}[\cos(\bar{g}_{\mathrm{DG}}, \nabla J)]}
     {1 - \mathbb{E}[\cos(\bar{g}_{\mathrm{PG}}, \nabla J)]}
\approx \frac{w_-^2 \cdot \mathrm{Var}_\perp(g_{\mathrm{PG}}) / (s^2 \|\nabla J\|^2)}
             {\mathrm{Var}_\perp(g_{\mathrm{PG}}) / \|\nabla J\|^2}
= \frac{w_-^2}{s^2}. \qedhere
\]

\subsection{Extension to Non-Symmetric Bandits}
\label{app:nonsymmetric}

The symmetric assumption in Proposition~\ref{prop:variance} gives clean equalities, but the noise-suppression mechanism extends to general policies.
The coincidence of oracles does not require symmetry: in any single-context bandit, the score identity $\sum_a \pi(a)\,\phi(a) = 0$ forces $\mathbb{E}[g_{\mathrm{PG}}] = \pi(y^*)\,\phi(y^*) \propto g^*_{\mathrm{CE}}$.
The directional improvement of Proposition~\ref{prop:direction} is therefore a fundamentally multi-context phenomenon; in a single context, DG can only reduce variance.

\paragraph{Variance suppression.}
For any incorrect action $a$ with negative advantage ($U < 0$) and surprisal $\ell(a) = -\log\pi(a)$, the DG gate satisfies
\[
w(a) = \sigma(-b\,\ell(a)/\eta) \;\le\; e^{-b\,\ell(a)/\eta} = \pi(a)^{b/\eta}.
\]
For rare actions ($\pi(a) \ll 1$), this bound is small: the gate suppresses their contribution by at least $\pi(a)^{b/\eta}$.
In a general (non-symmetric) bandit, each incorrect action $a$ contributes $\pi(a)\,w(a)^2\,b^2\,\|\Pi_\perp(\phi(a))\|^2$ to $\mathrm{Var}_\perp(g_{\mathrm{DG}})$.
Using the bound above:
\[
\mathrm{Var}_\perp(g_{\mathrm{DG}})
\;\le\; \sum_{a \neq y^*} \pi(a)^{1+2b/\eta}\,b^2\,\|\Pi_\perp(\phi(a))\|^2.
\]
When $b/\eta > 0$, the exponent $1 + 2b/\eta > 1$ ensures that rare actions are suppressed more aggressively than under PG (which has exponent~$1$).
The exact equality of Proposition~\ref{prop:variance}(ii) becomes an inequality, but the qualitative conclusion---DG suppresses perpendicular noise from rare incorrect actions---holds without symmetry.

\paragraph{Directional bias.}
Without symmetry, each incorrect action receives a different gate value, so $\mathbb{E}[g_{\mathrm{DG}}]$ is no longer exactly parallel to $\nabla J$.
This introduces a small perpendicular bias---a directional cost, since the two oracles already coincide in a single context.
The bias vanishes as the policy concentrates on the correct action ($\varepsilon \to 0$): all incorrect-action contributions shrink, and the gate values converge.

\subsection{Proof of Lemma~\ref{lem:greedy} (Greedy Directions)}
\label{app:greedy_lemma}

We derive the greedy-optimal direction for each objective by Cauchy--Schwarz.
Let $a_n := \|v_n\|^2$ then for an update $z \leftarrow z + \alpha\, g/\|g\|$ with $g = \sum_n c_n v_n$ and $\|g\|^2 = \sum_n c_n^2 a_n$ (by orthogonality):

\paragraph{Arithmetic objective.}
$\Delta(\sum_n p_n) = \sum_n \langle \nabla_{z_n} p_n, \alpha c_n v_n / \|g\| \rangle$.
Since $\nabla_{z_n} p_n = p_n v_n$, each term is $\alpha c_n p_n a_n / \|g\|$.
So $\Delta(\sum p_n) = \alpha \sum_n c_n p_n a_n / \sqrt{\sum_m c_m^2 a_m}$.

To maximize over $c$, set $x_n = c_n \sqrt{a_n}$ and $y_n = p_n \sqrt{a_n}$.
The ratio $\sum x_n y_n / \|x\|$ is maximized by Cauchy--Schwarz when $x \propto y$, i.e., $c_n \propto p_n$.

\paragraph{Log objective.}
$\Delta(\sum_n \log p_n) = \sum_n \langle \nabla_{z_n} \log p_n, \alpha c_n v_n / \|g\| \rangle$.
Since $\nabla_{z_n} \log p_n = v_n$, each term is $\alpha c_n a_n / \|g\|$.
Maximized when $c_n \propto a_n / a_n = 1$ (again by Cauchy--Schwarz with $y_n = \sqrt{a_n}$). \qed

\subsection{Proof of Proposition~\ref{prop:direction} ($N = 2$)}
\label{app:prop_b}

The key tool is a monotonicity lemma: the cosine between a weighted sum and the equal-weight sum is uniquely maximized when the ratio of weights equals one.

\begin{lemma}[Two-vector cosine monotonicity]
\label{lem:cosine_mono}
For $v_1, v_2$ orthogonal with norms $a_1, a_2 > 0$, define $C(r) := \cos(r\,v_1 + v_2,\; v_1 + v_2)$ for $r > 0$.
Then $C(r)$ is uniquely maximized at $r = 1$ and strictly decreasing as $r$ moves away from $1$.
\end{lemma}

\begin{proof}
By orthogonality:
\[
C(r) = \frac{r\,a_1 + a_2}{\sqrt{(r^2 a_1 + a_2)(a_1 + a_2)}}.
\]
Since $C(r) > 0$ on $(0,\infty)$, $C$ is maximized where $C(r)^2$ is.
Define $Q(r) := C(r)^2 \cdot (a_1 + a_2) = (r\,a_1 + a_2)^2 / (r^2 a_1 + a_2)$.
Differentiating by the quotient rule and simplifying the numerator:
\[
Q'(r) = \frac{2a_1 a_2(1 - r)(r\,a_1 + a_2)}{(r^2 a_1 + a_2)^2}.
\]
Since $a_1, a_2 > 0$ and $r\,a_1 + a_2 > 0$, we have $Q'(r) > 0$ for $r < 1$ and $Q'(r) < 0$ for $r > 1$.
Thus $Q$ (and hence $C$) is uniquely maximized at $r = 1$.
\end{proof}

\paragraph{Proof of Proposition~\ref{prop:direction}.}
The PG population direction has ratio $r_{\mathrm{PG}} = p_1/p_2$.
The DG population direction has ratio $r_{\mathrm{DG}} = h(p_1)/h(p_2)$, where $h(p) = p\,\sigma((-\log p)/\eta)$.

Assume WLOG $p_1 > p_2$, so $r_{\mathrm{PG}} > 1$.
Since $h$ is increasing, $r_{\mathrm{DG}} = h(p_1)/h(p_2) > 1$.
Since the gate $\sigma((-\log p)/\eta)$ is decreasing, $r_{\mathrm{DG}} = (p_1/p_2) \cdot [\sigma((-\log p_1)/\eta)/\sigma((-\log p_2)/\eta)] < r_{\mathrm{PG}}$.
Thus $1 < r_{\mathrm{DG}} < r_{\mathrm{PG}}$, and by Lemma~\ref{lem:cosine_mono}, $C(r_{\mathrm{DG}}) > C(r_{\mathrm{PG}})$. \qed

\subsection{Directional Improvement for General $N$}
\label{app:prop_b_general}

The $N = 2$ proof (Appendix~\ref{app:prop_b}) uses direct ratio compression.
For general $N$, we interpolate continuously between PG and DG weights and show the squared cosine strictly increases along the path.

\begin{proposition}[Directional improvement, general $N$]
\label{prop:direction_general}
Let $N \ge 2$ contexts with orthogonal score vectors $v_n$ of norms $a_n > 0$.
Let $\eta > 1/2$, and suppose the $p_n$ are not all equal.
Then
\[
\cos\!\big(\mathbb{E}[g_{\mathrm{DG}}],\; g_{\mathrm{CE}}\big)
\;>\;
\cos\!\big(\mathbb{E}[g_{\mathrm{PG}}],\; g_{\mathrm{CE}}\big).
\]
\end{proposition}

\begin{proof}
Define the interpolated weights $c_n(t) = p_n\,[\sigma((-\log p_n)/\eta)]^t$ for $t \in [0,1]$.
At $t = 0$, $c_n(0) = p_n$ (PG weights); at $t = 1$, $c_n(1) = p_n \,\sigma((-\log p_n)/\eta)$ (DG weights).
Define the squared-cosine proxy
\[
\Phi(t) := \frac{\big(\sum_n c_n(t)\,a_n\big)^2}{\sum_n c_n(t)^2\,a_n},
\]
which is proportional to $\cos^2\!\big(\sum c_n(t)\,v_n,\;\sum v_n\big)$.
We show $\Phi'(t) > 0$ for all $t \in [0,1)$.

Setting $\lambda_n := \log \sigma((-\log p_n)/\eta) < 0$ (since $\sigma < 1$ on $(0,1)$), we have $c_n'(t) = c_n(t)\,\lambda_n$.
Differentiating $\Phi = A^2/B$ with $A = \sum c_n a_n$ and $B = \sum c_n^2 a_n$:
\[
\Phi'(t) = \frac{2A}{B^2}\Big[A'\,B - A\,B'\!/2\Big] \;\propto\; \sum_{n,m} c_n\,c_m^2\,a_n\,a_m\,(\lambda_n - \lambda_m).
\]
Antisymmetrizing:
\begin{equation}
\label{eq:path_sign}
\Phi'(t) \;\propto\; \sum_{n < m} (\lambda_n - \lambda_m)\,a_n\,a_m\,c_n\,c_m\,(c_m - c_n).
\end{equation}
We check the sign of each term.
Since $\eta > 1/2$, the function $p \mapsto p\,[\sigma((-\log p)/\eta)]^t$ is increasing for all $t \in [0,1]$.%
\footnote{Write $p\,\sigma((-\log p)/\eta)^t = p/(1+p^{1/\eta})^t$.
Its derivative has the sign of $1 + (1-t/\eta)\,p^{1/\eta}$, which at worst ($t{=}1$, $p{\to}1$) equals $2 - 1/\eta > 0$.}
Therefore $c_n(t)$ preserves the ordering of the $p_n$.
Since $\lambda_n$ is strictly decreasing in $p_n$ (as $\sigma((-\log p)/\eta)$ is decreasing), for each pair $n < m$ with $p_n \neq p_m$:
\begin{itemize}
\item $p_n > p_m \implies c_n > c_m$ and $\lambda_n < \lambda_m$, so $(\lambda_n - \lambda_m)(c_m - c_n) > 0$.
\item $p_n < p_m \implies c_n < c_m$ and $\lambda_n > \lambda_m$, so $(\lambda_n - \lambda_m)(c_m - c_n) > 0$.
\end{itemize}
Each factor $a_n\,a_m\,c_n\,c_m$ is positive, so every non-degenerate term in~\eqref{eq:path_sign} is strictly positive.
Since not all $p_n$ are equal, $\Phi'(t) > 0$ for all $t \in [0,1)$, giving $\Phi(1) > \Phi(0)$.
Both cosines are positive, so $\cos(\mathbb{E}[g_{\mathrm{DG}}], g_{\mathrm{CE}}) > \cos(\mathbb{E}[g_{\mathrm{PG}}], g_{\mathrm{CE}})$.
\end{proof}

\paragraph{On the $\eta > 1/2$ condition.}
The condition $\eta > 1/2$ ensures that the DG weight function $p \mapsto p\,\sigma((-\log p)/\eta)$ is monotonically increasing, which the path argument requires to preserve the ordering of weights.
Since we use $\eta = 1$ throughout, this condition is always satisfied.
For $\eta \le 1/2$ the weight function can be non-monotone near $p = 1$, and Proposition~\ref{prop:direction} may fail; exploring this regime is an interesting direction for future work.

\section{Transformer Sequence Modeling}
\label{app:scaling}

We provide experimental details and robustness checks for the Token Reversal experiments in Section~\ref{sec:transformer}.

\paragraph{Experimental Setup.}
The agent is a decoder-only Transformer with causal attention, model dimension $d_{\text{model}}=64$, 2 layers, and 2 attention heads.
We use a distributed actor-learner architecture with 10 parallel actors, each collecting batches of 10 trajectories, yielding 100 episodes per gradient step.
All agents are trained with Adam using a standard empirical mean baseline for variance reduction.
Unless otherwise noted, the training budget is $K=1{,}000$ episodes.

\subsection{Baseline Tuning}
\label{app:scaling_tuning}

To ensure fair comparison, we tuned hyperparameters for PPO and PMPO.
For PPO, we swept the clipping parameter $\epsilon$ and KL penalty $\beta$; Figure~\ref{fig:baseline_tune}a shows the algorithm is insensitive to $\epsilon$, and the KL penalty does not help on this deterministic task.
For PMPO, we swept the weighting threshold $\alpha$ and KL penalty $\beta$; Figure~\ref{fig:baseline_tune}b shows boundary values ($\alpha \approx 0$ or $1$) outperform intermediate values.
Even the best-tuned PPO ($\text{Regret} \approx 0.03$) and PMPO ($\text{Regret} \approx 0.08$) lag behind DG ($\text{Regret} < 0.01$).

\begin{figure}[ht]
\centering
\begin{subfigure}[b]{0.48\columnwidth}
    \includegraphics[width=\linewidth]{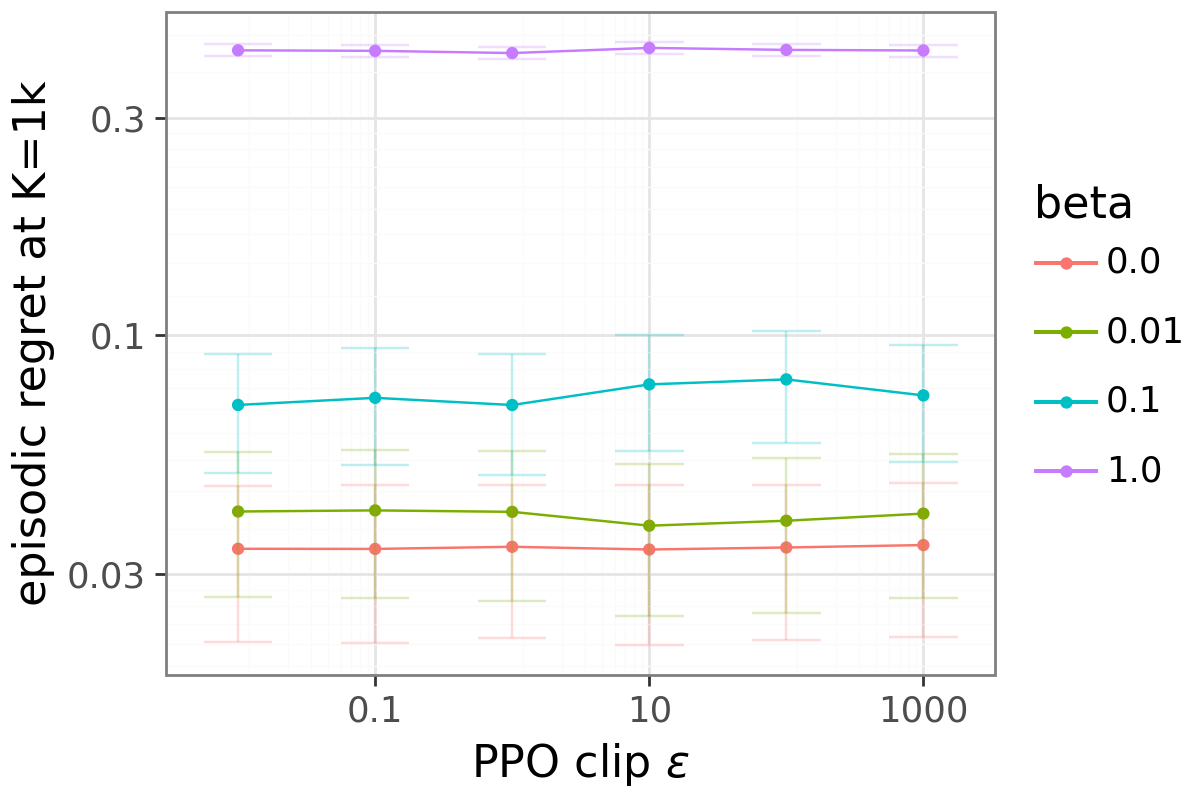}
    \caption{Tuning PPO ($\epsilon$, $\beta$).}
\end{subfigure}
\hfill
\begin{subfigure}[b]{0.48\columnwidth}
    \includegraphics[width=\linewidth]{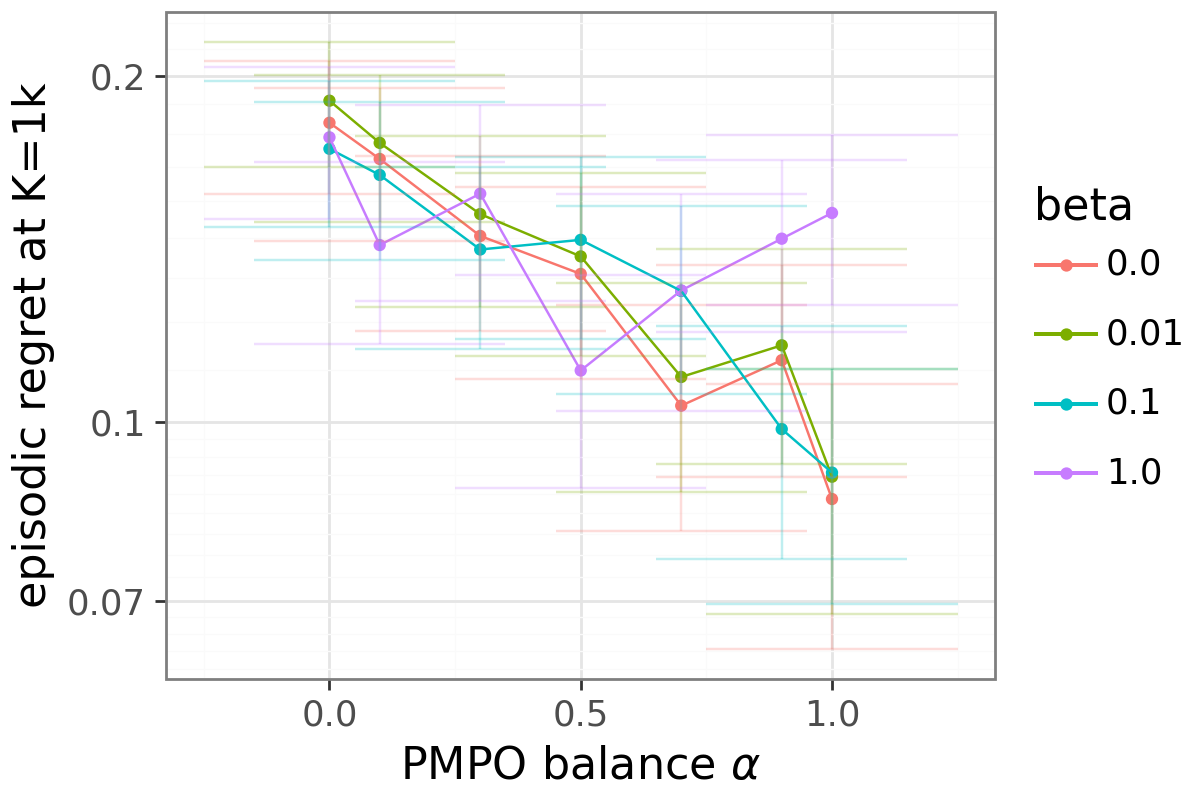}
    \caption{Tuning PMPO ($\alpha$, $\beta$).}
\end{subfigure}
\caption{Hyperparameter tuning for baselines. Neither PPO nor PMPO matches DG despite extensive sweeps.}
\label{fig:baseline_tune}
\end{figure}

\subsection{Robustness}
\label{app:scaling_robustness}

We validate robustness on the default task ($H=10$, $M=2$).
Figure~\ref{fig:combined_ablation}a shows regret at $K=1{,}000$ across learning rates; DG consistently outperforms baselines over a wide effective range.
Figure~\ref{fig:combined_ablation}b extends training to $K=10{,}000$ episodes; baselines do not catch up, confirming the advantage is not due to faster early learning alone.

\begin{figure}[ht]
\centering
\begin{subfigure}[b]{0.48\columnwidth}
    \includegraphics[width=\linewidth]{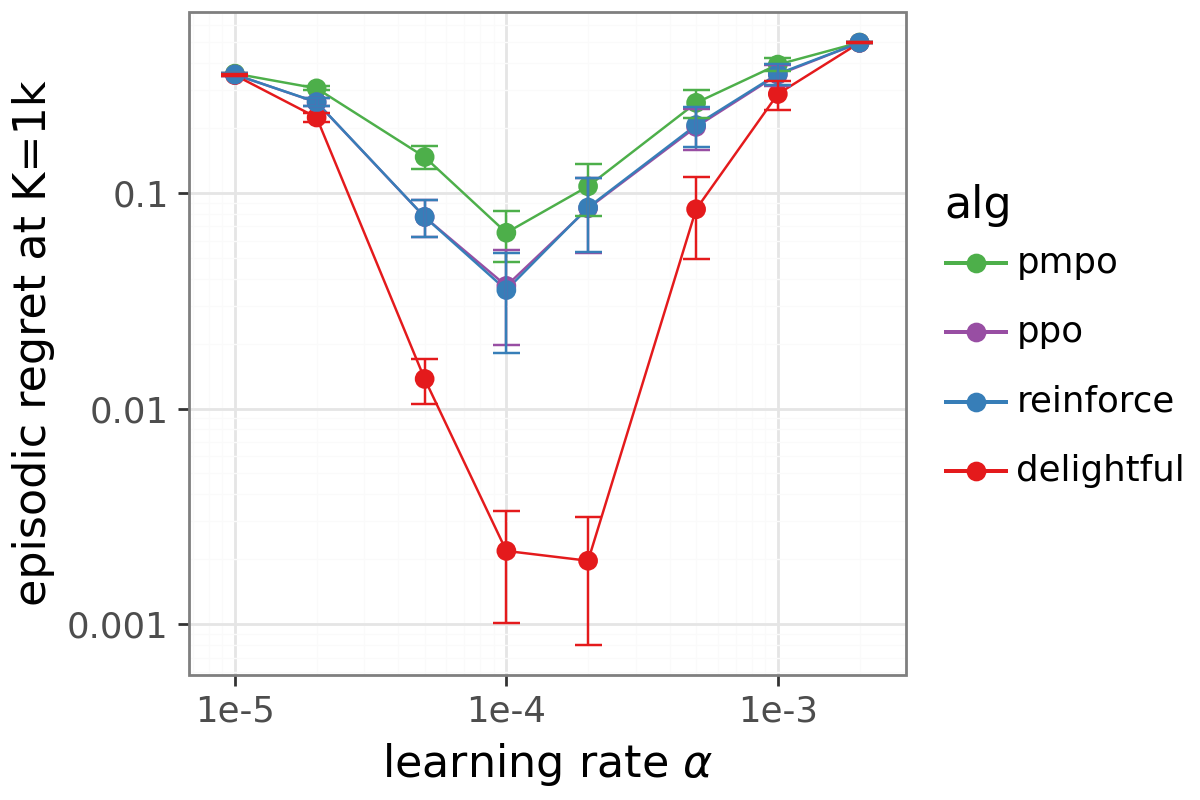}
    \caption{Learning rate sensitivity.}
\end{subfigure}
\hfill
\begin{subfigure}[b]{0.48\columnwidth}
    \includegraphics[width=\linewidth]{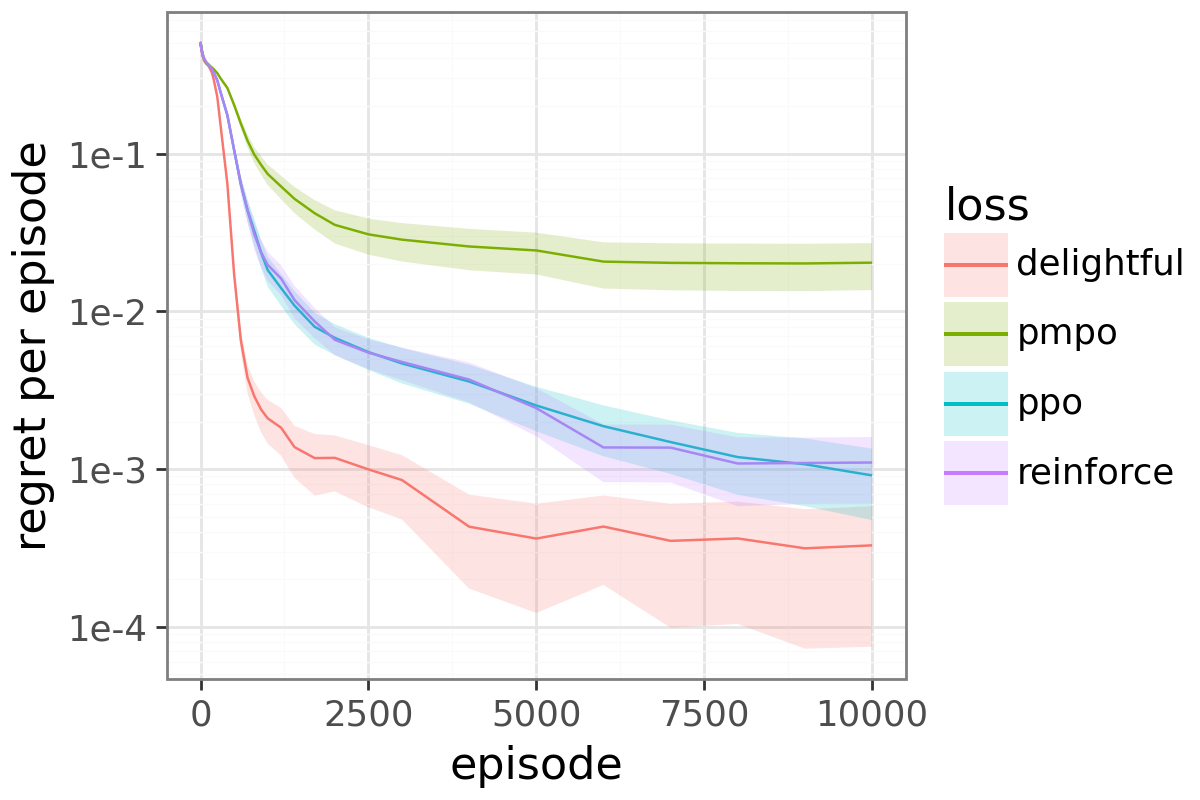}
    \caption{Extended training ($K=10\text{k}$).}
\end{subfigure}
\caption{DG's advantage is robust to learning rate \textbf{(a)} and persists asymptotically \textbf{(b)}.}
\label{fig:combined_ablation}
\end{figure}

\subsection{Task Variations}
\label{app:scaling_environment}

To ensure findings generalize beyond reversal, we test four target logics:
\begin{itemize}
    \item \textbf{Copy:} $y_i = x_i$
    \item \textbf{Flip:} $y_i = 1 - x_i$
    \item \textbf{Reverse Copy:} $y_i = x_{H-i+1}$ (the default)
    \item \textbf{Reverse Flip:} reverse and negate
\end{itemize}
We also vary reward structure.
\textbf{Bag-of-Tokens} gives credit for each correct token regardless of position; \textbf{Sequential} gives credit only up to the first mistake, making credit assignment harder.
Figure~\ref{fig:env_structure} illustrates the difference.

Figure~\ref{fig:env_regret} shows learning curves for all eight configurations.
DG achieves the lowest regret in every setting.
The gap is often larger in the harder Sequential settings, where baselines struggle with truncated reward streams.

\begin{figure}[ht]
\centering
\begin{subfigure}[b]{0.48\columnwidth}
    \includegraphics[width=\linewidth]{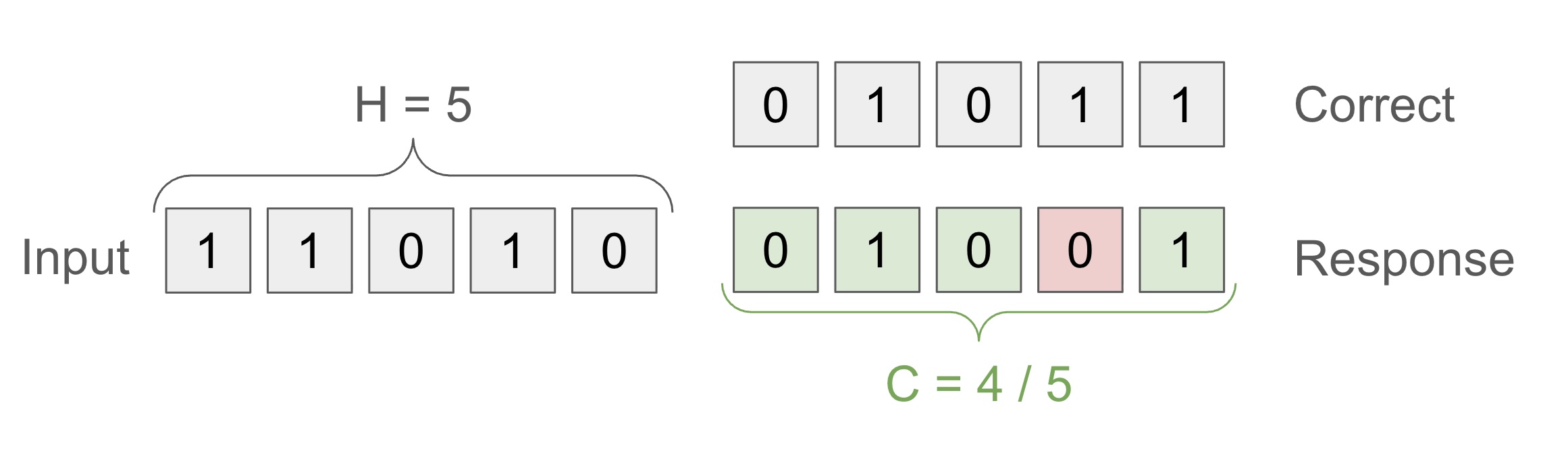}
    \caption{Bag-of-Tokens (dense).}
\end{subfigure}
\hfill
\begin{subfigure}[b]{0.48\columnwidth}
    \includegraphics[width=\linewidth]{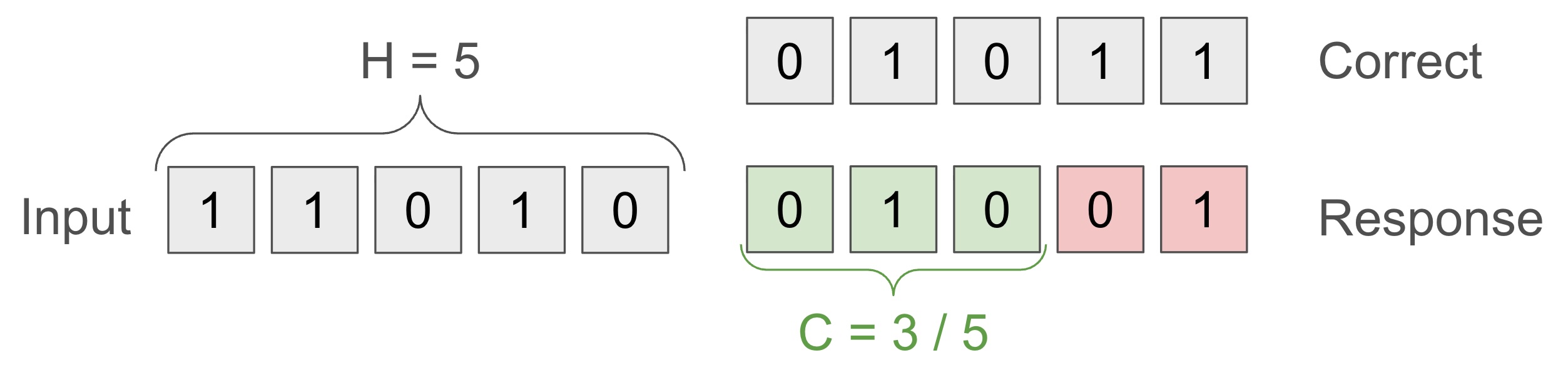}
    \caption{Sequential (strict).}
\end{subfigure}
\caption{Reward structures. Bag-of-Tokens credits all correct tokens; Sequential stops at the first error.}
\label{fig:env_structure}
\end{figure}

\begin{figure}[ht]
\centering
\includegraphics[width=\columnwidth]{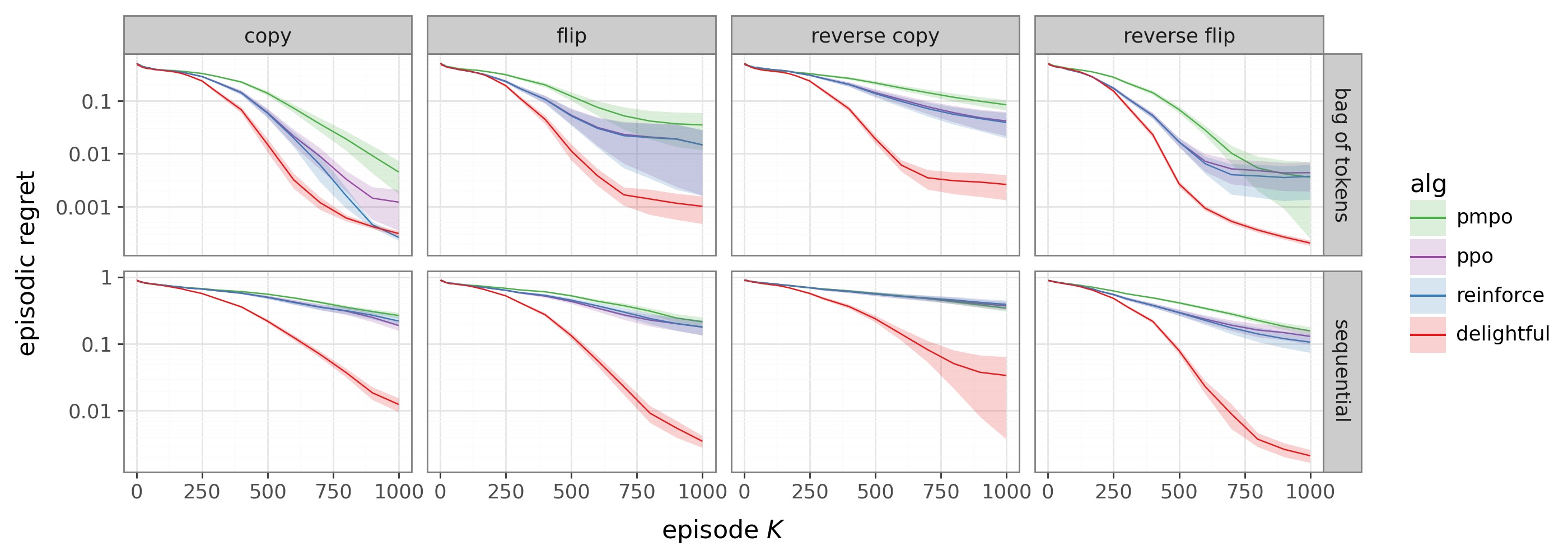}
\caption{Learning curves across 8 task variants. Top: Bag-of-Tokens. Bottom: Sequential. DG (red) achieves the lowest error in all configurations.}
\label{fig:env_regret}
\end{figure}

\subsection{Multiplicative vs.\ Additive Delight}
\label{app:scaling_alternative}

We compare DG to UCB-style additive mixtures:
\begin{equation}
    \chi^{\rm UCB}_\alpha = (1-\alpha) U + \alpha \ell.
\end{equation}
Figure~\ref{fig:compare_ucb} sweeps $\alpha \in [0, 1.25]$ and $\eta \in \{0.2, 0.5, 1, 2, 5\}$.
Additive mixtures can outperform REINFORCE (black dashed) but never approach DG (red dashed).
The best additive achieves regret $\approx 0.1$ versus DG's $\approx 0.04$.
Additive bonuses treat surprisal symmetrically, encouraging the policy to chase high-surprisal actions even when advantage is negative.
DG's multiplicative gate suppresses such blunders, filtering them from the update.

\begin{figure}[ht]
\centering
\includegraphics[width=0.55\columnwidth]{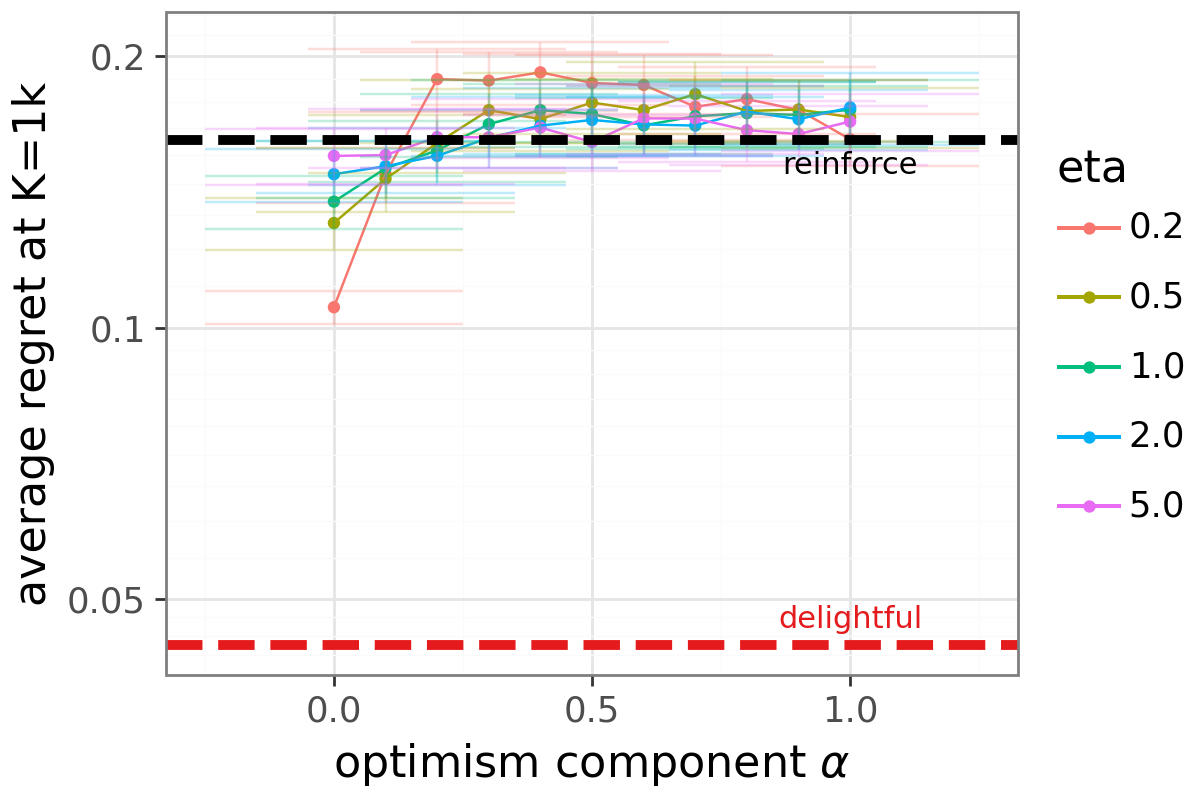}
\caption{Additive mixtures (colored lines) vs.\ multiplicative DG (red dashed). No additive configuration matches DG.}
\label{fig:compare_ucb}
\end{figure}

\section{Detailed Control Suite Results}
\label{app:control_full}

We provide experimental details and extended results for Section~\ref{sec:control}.

\subsection{Experimental Setup}
\label{app:control_setup}

\paragraph{Architecture.}
All methods use the same actor-critic architecture: a 2-layer MLP with 256 hidden units for both actor and critic.
The actor outputs mean and diagonal covariance of a Gaussian policy; the critic estimates state-action values.
We use the Retrace algorithm~\citep{munos2016safe} for off-policy correction.

\paragraph{Optimization.}
All methods use Adam with learning rate $3 \times 10^{-4}$.
We train for 10 million environment steps with 4 parallel actors, a replay buffer of size $2 \times 10^6$, and batch size 256.
Target networks are updated every 100 learner steps.

\paragraph{DG-specific details.}
For continuous actions, surprisal $\ell = -\log \pi(a \mid s)$ can take large values.
We clip surprisal to $[-10, 10]$ before computing delight.
We also normalize rewards using an exponential moving average (decay 0.999) for critic stability; this normalization is applied to all methods.
The gate temperature is $\eta = 1$, consistent with all other experiments.

Algorithm~\ref{alg:dg_continuous} provides the continuous-action variant used in all control experiments.

\begin{algorithm}[ht]
\caption{Delightful Policy Gradient (continuous actions)}
\label{alg:dg_continuous}
\begin{algorithmic}[1]
\Require Batch $\mathcal{B}$, Gaussian policy $\pi_\theta(\cdot \mid s)$, temperature $\eta{=}1$, clip bound $C{=}10$
\State $\Delta\theta \gets 0$
\For{$t \in \mathcal{B}$}
    \State $\ell_t \gets \mathrm{clip}(-\log \pi_\theta(A_t \mid \mathcal{H}_t),\; -C,\; C)$
    \State $\chi_t \gets U_t \cdot \ell_t$ \Comment{Delight}
    \State $w_t \gets \sigma(\chi_t / \eta)$ \Comment{Gate}
    \State $\Delta\theta \gets \Delta\theta + w_t \, U_t \, \nabla_\theta \log \pi_\theta(A_t \mid \mathcal{H}_t)$
\EndFor
\State \Return $\Delta\theta$
\end{algorithmic}
\end{algorithm}

\paragraph{Baselines.}
We compare against three baselines within our codebase:
\begin{itemize}
    \item \textbf{PG:} Standard policy gradient with Retrace critic and no gating.
    \item \textbf{PPO:} Clipped surrogate objective with $\epsilon = 0.2$.
    \item \textbf{MPO:} Softmax-weighted updates with temperature $\eta = 1.0$.
\end{itemize}
All baselines use identical architecture, optimizer, and replay settings.

\subsection{Per-Environment Learning Curves}
\label{app:control_curves}

Figure~\ref{fig:reward_control} displays individual learning curves for all 28 Control Suite environments.
DG (red) consistently matches or exceeds baseline performance, with notable improvements on exploration-heavy tasks such as \texttt{acrobot:swingup}, \texttt{finger:turn\_hard}, and \texttt{humanoid:run}.

\begin{figure*}[h]
    \centering
    \includegraphics[width=\textwidth]{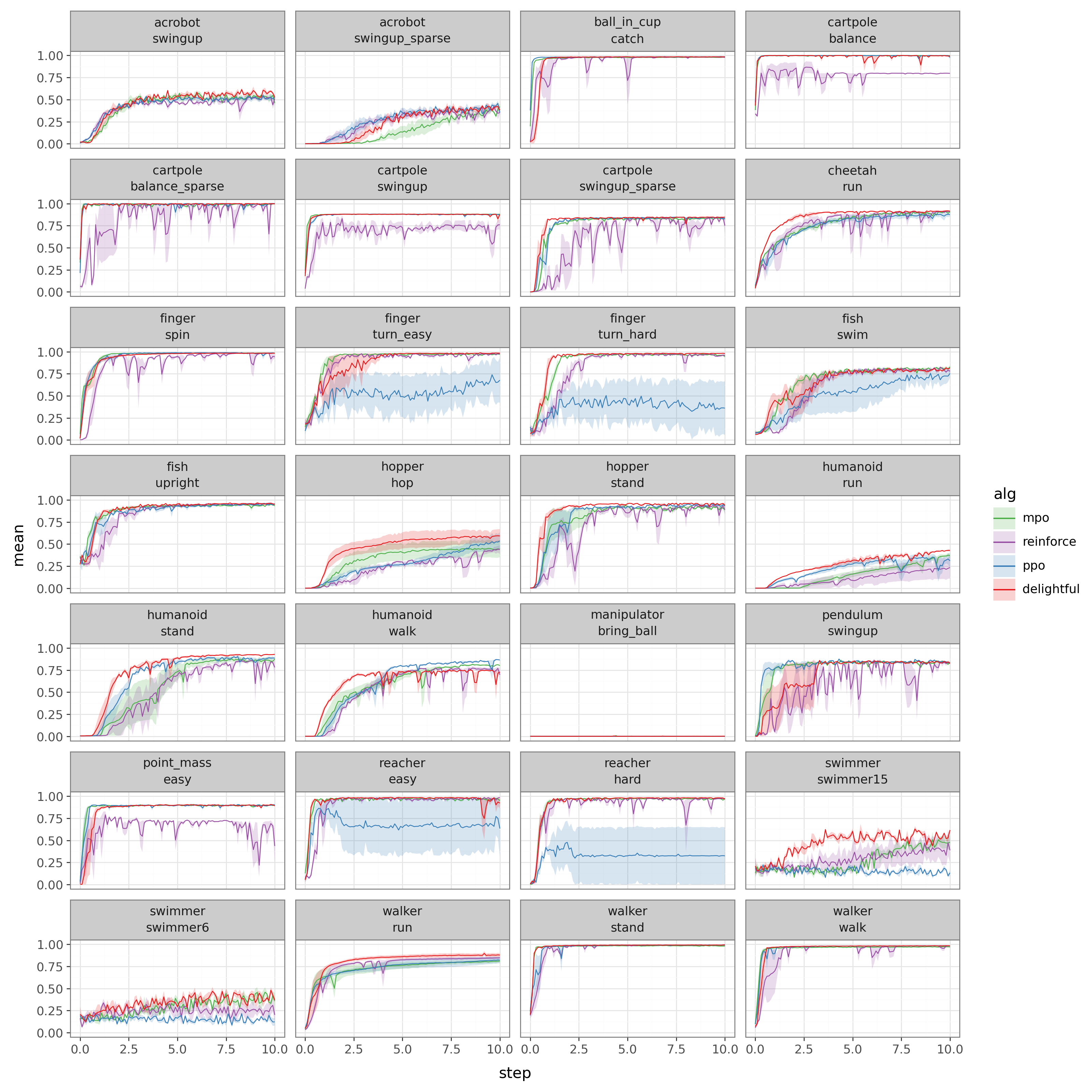}
    \caption{Learning curves for 28 Control Suite environments.
    DG (red) consistently matches or exceeds the hedonic baseline (purple), PPO (blue), and MPO (green).}
    \label{fig:reward_control}
\end{figure*}

\subsection{Baseline Hyperparameter Sensitivity}
\label{app:control_tuning}

To ensure fair comparison, we verified that default hyperparameters are reasonable for PPO and MPO.
Figure~\ref{fig:control_tuning} shows sensitivity sweeps on \texttt{cartpole:swingup}.
For PPO, we sweep clip parameter $\epsilon \in [0.01, 100]$; for MPO, we sweep temperature $\eta \in [0.001, 1000]$.
Neither sweep reveals a clear optimum that substantially outperforms the defaults ($\epsilon = 0.2$, $\eta = 1.0$).
We therefore use these standard values throughout.

\begin{figure}[h]
    \centering
    \begin{subfigure}[b]{0.48\columnwidth}
        \includegraphics[width=\linewidth]{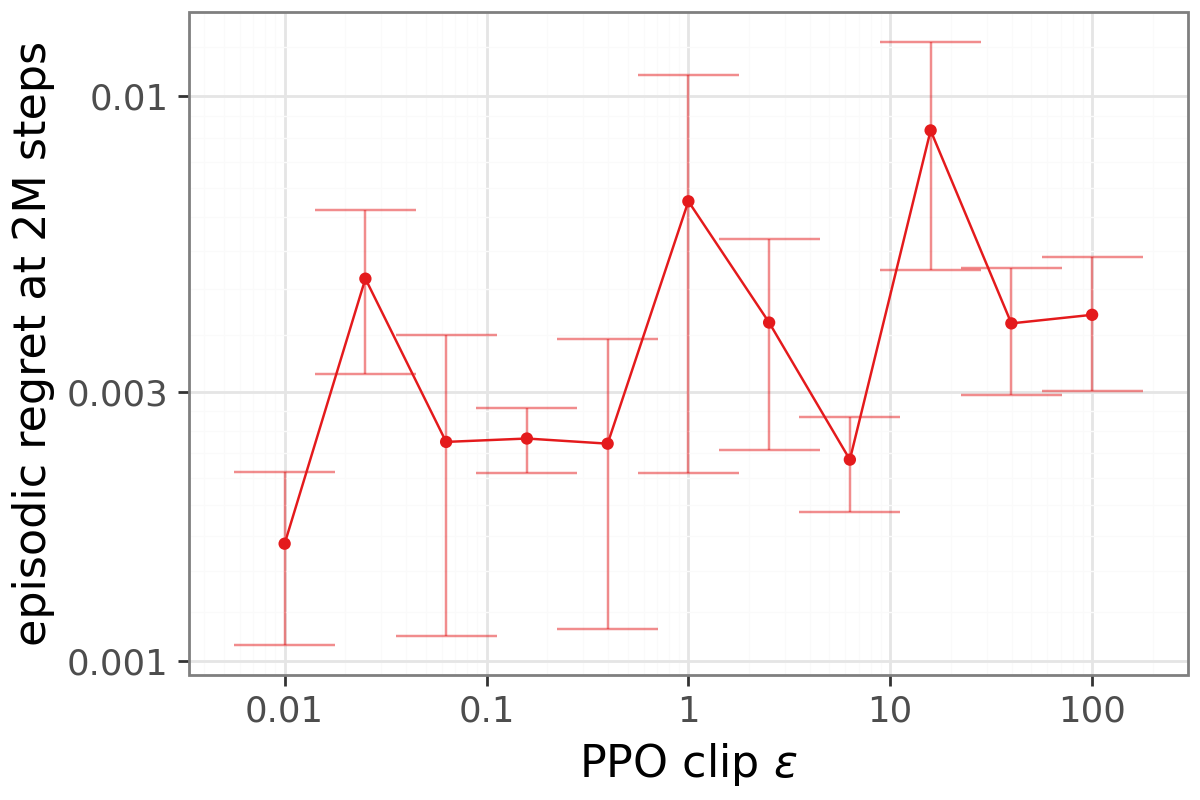}
        \caption{PPO clip $\epsilon$.}
        \label{fig:cartpole_tune_ppo}
    \end{subfigure}
    \hfill
    \begin{subfigure}[b]{0.48\columnwidth}
        \includegraphics[width=\linewidth]{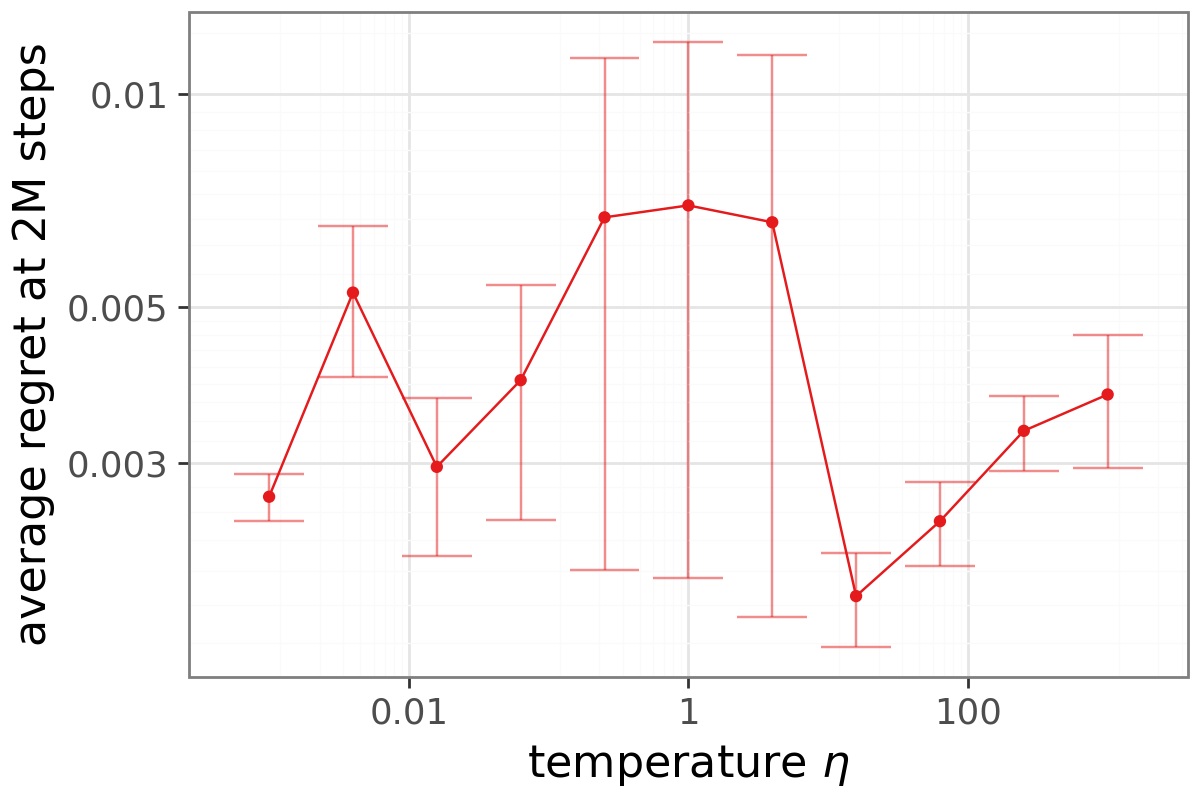}
        \caption{MPO temperature $\eta$.}
        \label{fig:cartpole_tune_mpo}
    \end{subfigure}
    \caption{Hyperparameter sensitivity on \texttt{cartpole:swingup}. 
    Default values (dashed) perform comparably to alternatives.}
    \label{fig:control_tuning}
\end{figure}

\subsection{Comparison to Tuned External Implementations}
\label{app:control_sota}

We also benchmark DG against highly-optimized external implementations: \textbf{MPO (Tuned)} with adaptive temperature and \textbf{SAC (Tuned)} with automatic entropy adjustment.
These use separate codebases with extensive per-task tuning.
This comparison is especially stringent because regret is defined relative to the best performance achieved by \emph{any} method: $\text{Regret}_k = R_{\text{best}} - R_k$.

Figure~\ref{fig:regret_ave_control_tuned} shows aggregate regret over training.
Even against these optimized baselines, DG achieves the lowest average regret.
Figure~\ref{fig:regret_env_control} breaks down final performance by environment, confirming that DG's advantage is broad-based rather than driven by outliers.
DG achieves low regret across domains ranging from \texttt{cartpole} to \texttt{humanoid} and \texttt{dog}.

\begin{figure}[h]
    \centering
    \includegraphics[width=0.55\columnwidth]{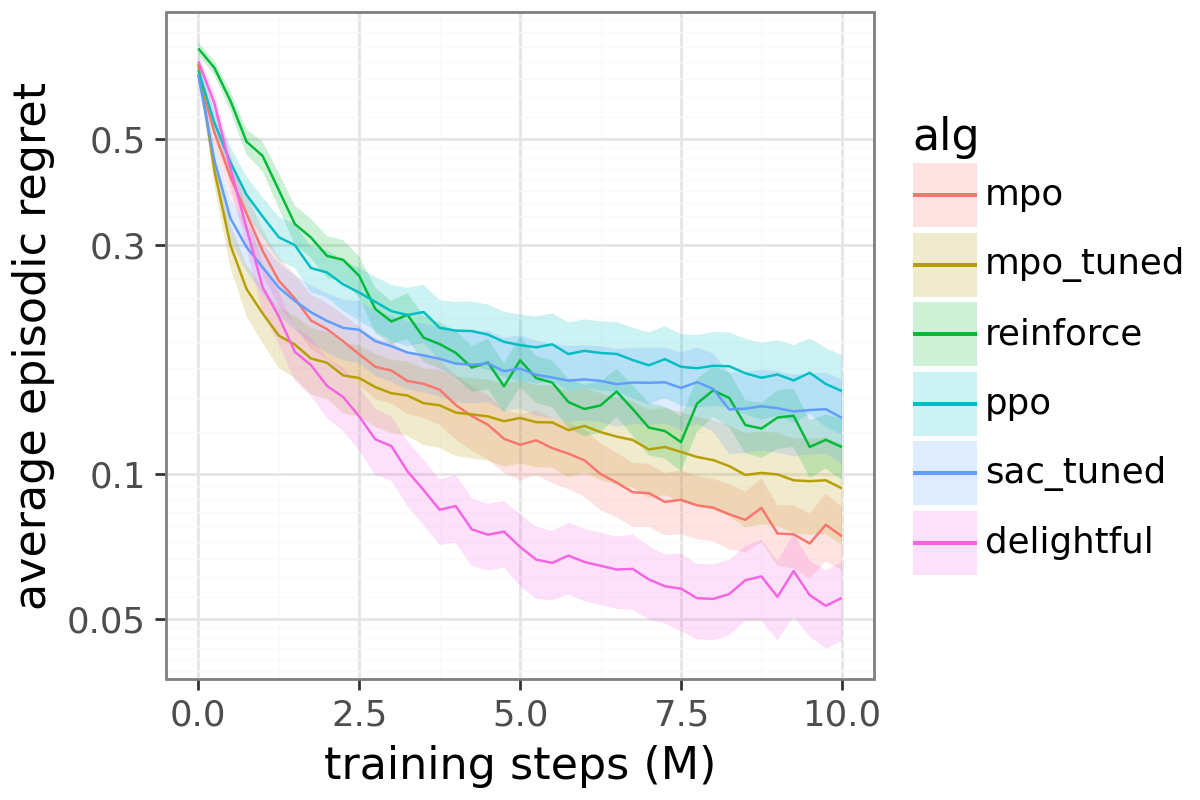}
    \caption{Aggregate regret against tuned SOTA baselines.
    DG (pink) outperforms tuned MPO (gold) and SAC (blue).}
    \label{fig:regret_ave_control_tuned}
\end{figure}

\begin{figure*}[h]
    \centering
    \includegraphics[width=\textwidth]{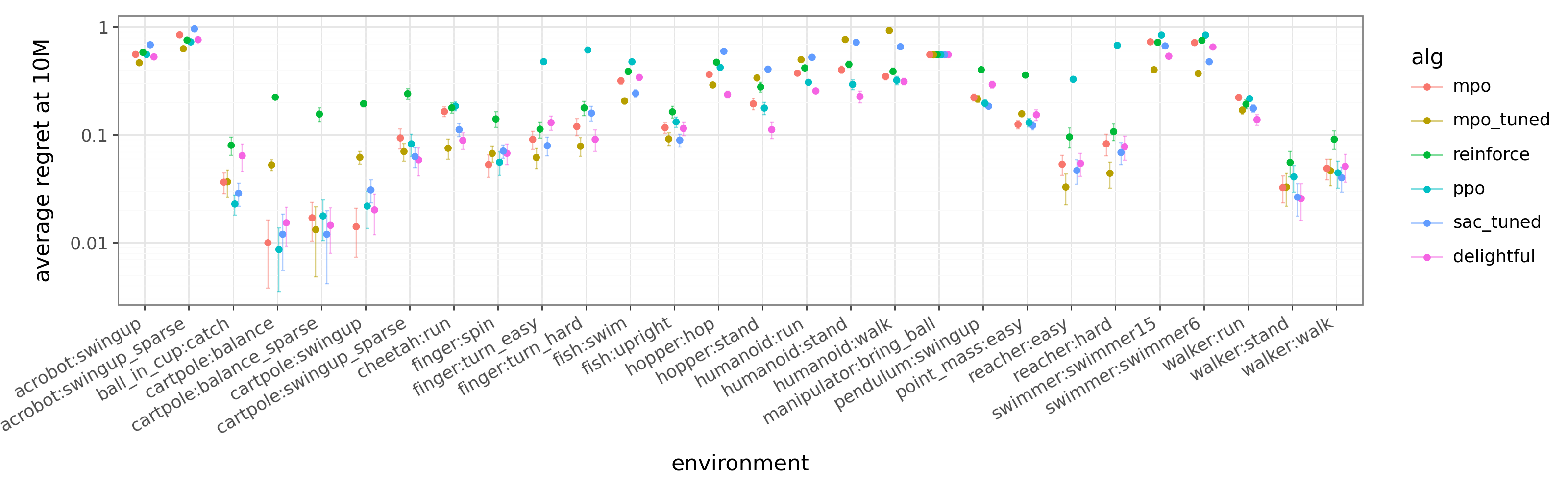}
    \caption{Per-environment regret at 10M steps.
    DG (pink) consistently achieves low regret across the suite.}
    \label{fig:regret_env_control}
\end{figure*}

\end{document}